%% file: main.tex
\documentclass[sigconf]{acmart}
\AtBeginDocument{%
  }

\copyrightyear{2026}
\acmYear{2026}
\setcopyright{cc}
\setcctype{by}
\acmConference[KDD '26]{Proceedings of the 32nd ACM SIGKDD Conference on Knowledge Discovery and Data Mining V.1}{August 09--13, 2026}{Jeju Island, Republic of Korea}
\acmBooktitle{Proceedings of the 32nd ACM SIGKDD Conference on Knowledge Discovery and Data Mining V.1 (KDD '26), August 09--13, 2026, Jeju Island, Republic of Korea}
\acmPrice{}
\acmDOI{10.1145/3770854.3780219}
\acmISBN{979-8-4007-2258-5/2026/08}
\usepackage{enumitem}
\usepackage{algorithm}
\usepackage{algorithmic}
\usepackage{amsmath,amsfonts}
\usepackage{amsthm}

\usepackage{amsthm}  
\theoremstyle{remark}  
  
\usepackage{graphicx}
\usepackage{booktabs}
\usepackage{multirow}
\usepackage{adjustbox}
\usepackage{xparse}
\usepackage{bm}
\usepackage{subcaption}

\NewDocumentCommand{\haoyu}{mg}{%
  \IfNoValueTF{#2}%
    {\textcolor{teal}{#1}}%
    {\textcolor{teal}{$\blacktriangleright$}#1
     \textcolor{teal}{$\triangleright$ #2$\blacktriangleleft$}}%
}
\newcommand{\ourmethod}{\textsc{SSTGNN}}
\newcommand{\mypara}[1]{
    \vspace{2pt}
    \par\noindent\textbf{#1}
    \vspace{0pt}
    \noindent}

\begin{document}

\title{When Deepfake Detection Meets Graph Neural Network: \\ a Unified and Lightweight Learning Framework}

\author{Haoyu Liu}\authornote{Both authors contributed equally to this research.}
\orcid{0000-0002-0839-5460}
\affiliation{%
  \institution{Nanyang Technological University}
  \country{Singapore}}
\email{haoyu.liu@ntu.edu.sg}

\author{Chaoyu Gong}\authornotemark[1]
\orcid{0000-0002-5540-5350}
\affiliation{%
  \institution{Nanyang Technological University}
  \country{Singapore}}
\email{chaoyu.gong@ntu.edu.sg}

\author{Mengke He}
\orcid{0009-0007-2623-7507}
\affiliation{%
  \institution{Nanyang Technological University}
  \country{Singapore}}
\email{mengke.he@ntu.edu.sg}

\author{Jiate Li}
\orcid{0009-0009-7829-0849}
\authornote{Work done when working as a research assistant at NTU.}
\affiliation{%
  \institution{University of Southern California}
  \country{Los Angeles, USA}}
\email{jiateli@usc.edu}

\author{Kai Han}
\orcid{0000-0002-7995-9999}
\affiliation{%
  \institution{The University of Hong Kong}
  \country{Hong Kong SAR}}
\email{kaihanx@hku.hk}

\author{Siqiang Luo}
\orcid{0000-0001-8197-0903}
\authornote{Siqiang Luo is the corresponding author.}
\affiliation{%
  \institution{Nanyang Technological University}
  \country{Singapore}}
\email{siqiang.luo@ntu.edu.sg}

\renewcommand{\shortauthors}{Liu et al.}

\begin{abstract}
    The proliferation of generative video models has made detecting AI-generated and manipulated videos an urgent challenge. Existing detection approaches often fail to generalize across diverse manipulation types due to their reliance on isolated spatial, temporal, or spectral information, and typically require large models to perform well. This paper introduces SSTGNN, a lightweight Spatial-Spectral-Temporal Graph Neural Network framework that represents videos as structured graphs, enabling joint reasoning over spatial inconsistencies, temporal artifacts, and spectral distortions. SSTGNN incorporates learnable spectral filters and spatial-temporal differential modeling into a unified graph-based architecture, capturing subtle manipulation traces more effectively. Extensive experiments on diverse benchmark datasets demonstrate that SSTGNN not only achieves superior performance in both in-domain and cross-domain settings, but also offers strong efficiency and resource allocation. Remarkably, SSTGNN accomplishes these results with up to \textbf{42$\times$} fewer parameters than state-of-the-art models, making it highly lightweight and resource-friendly for real-world deployment.
\end{abstract}



\maketitle
\newcommand\kddavailabilityurl{https://doi.org/10.5281/zenodo.18042443}
\ifdefempty{\kddavailabilityurl}{}{
\begingroup\small\noindent\raggedright\textbf{Resource Availability:}\\
The source code of this paper has been made publicly available at \url{\kddavailabilityurl} and \url{https://github.com/hyLiu-777/SSTGNN}.
\endgroup
}

\input{text/intro}
\input{text/related_work}
\input{floats/fig_framework}

\input{text/method}
\input{text/experiments}

\input{text/conclusion}
\input{text/acks}

\bibliographystyle{ACM-Reference-Format}
\bibliography{main}

\input{text/appendix}

\end{document}

%% file: text/intro.tex
\input{floats/fig_intro}
\section{Introduction}
\label{sec:intro}

The rapid advancement of generative video models, such as Sora~\cite{openai2024sora} and Runway Gen-2~\cite{runway2024gen2}, has significantly lowered the barrier for producing realistic synthetic videos. While these technologies enable creative applications, they also raise serious societal concerns, particularly when misused to manipulate visual evidence in sensitive domains like politics, finance, and public safety~\cite{nazren2024deepfake}. In response, detecting AI-generated videos has become a critical task for preserving media integrity and public trust.
Despite recent progress, the detection landscape remains challenging. Manipulated videos can exhibit diverse and subtle forms of tampering, including fine-grained appearance alterations, imperceptible frequency artifacts, or temporally incoherent frame transitions~\cite{wang2023dynamic}.



\mypara{Motivation.}  
As the boundaries between real and synthetic videos continue to blur, detecting manipulated content is no longer just a technical challenge but a fundamental requirement for safeguarding information integrity. Despite growing research attention, existing methods still struggle to generalize across diverse and unseen manipulation types, particularly those beyond facial forgeries. This calls for more versatile detection frameworks that can reason about manipulations at a deeper, more structural level. Therefore, our \underline{\textbf{first motivation}} is \textit{to address this urgent and socially significant problem with a more generalizable and holistic detection method}.

In this context, a critical question arises: how should we model the video content to enable such holistic understanding?  
Over the years, a wide range of deepfake detection methods have been proposed~\cite{wang2023dynamic}. Among them, convolutional architectures such as ResNet~\cite{he2016deep} and Xception~\cite{chollet2017xception} have been widely adopted to capture spatial and temporal (i.e., visual-only) inconsistencies between real and manipulated content~\cite{rossler2019faceforensics,afchar2018mesonet}. In parallel, frequency-domain approaches~\cite{luo2021generalizing,li2021frequency,masi2020two} aim to detect high-frequency artifacts, upsampling traces~\cite{liu2021spatial}, and spectral anomalies characteristic of GAN-generated content~\cite{goodfellow2014generative}. However, these methods typically process spatial, temporal, and spectral information in isolation, or fuse them via simple late-stage aggregation, without capturing their complex interactions. Moreover, these methods often rely on over-parameterized architectures to compensate for their limited representational capacity, resulting in large model sizes and reduced scalability. This limitation raises our \underline{\textbf{second motivation}}: \textit{the need for a unified and lightweight  framework capturing interactions across spatial, temporal, and spectral domains}.
In fact, graph-based representations offer a natural fit by treating a video as a spatiotemporal graph, where nodes represent local patch features and edges encode structured cross-domain relations. Graph Neural Networks (GNNs) built on such structures would learn expressive representations with significantly fewer parameters.
Figure \ref{motivation} illustrates quantitatively the trade-offs between accuracy, model size, and training cost across several state-of-the-art baselines. 
While some baselines achieve reasonable in-domain or cross-domain performance, they require larger models and more training time. In contrast, our GNN-based model achieves strong accuracy across domains with a significantly smaller model size (up to 42.4$\times$ fewer parameters) and lower training cost, as shown by the upper-left position and small marker size in both panels.

\mypara{Technical Challenges.} We here summarize three major obstacles to advancing the use of GNNs for deepfake video detection:

\noindent\textbf{(1) How to model expressive and efficient patch-level video graphs?}
Unlike previous approaches that represent each frame or object as a node, our objective is to capture fine-grained visual consistencies or inconsistencies at the patch level. This requires designing dense, high-resolution graphs where each node corresponds to a patch and edges encode both intra-frame and inter-frame relationships. The main challenge is to preserve local texture dependencies and temporal motion continuity, while maintaining tractable graph size and training efficiency.

\noindent\textbf{(2) How to learn adaptive spectral filters on video graphs?}
Traditional frequency-based methods employ fixed transforms (e.g., DCT), which are limited in adapting to diverse manipulation artifacts. Our framework instead defines learnable spectral filters over the eigen-decomposition of the video graph Laplacian. The technical challenge is to ensure these filters are expressive enough to capture structural distortions, yet generalizable across videos with varying topologies, graph sizes, and manipulation types.

\noindent\textbf{(3) How to encode non-local temporal and spatial artifacts in graphs?}
Subtle manipulations such as jitter, frame inconsistencies, or unnatural motion patterns are often temporally sparse and spatially dispersed. To detect these, we propose encoding frame-wise differences as negatively temporal edges between corresponding patches. The challenge lies in amplifying temporal inconsistencies without compromising spatial-spectral coherence or introducing instability into the graph learning process.

To address the above challenges, we summarize our key contributions as follows:

\begin{itemize}[leftmargin=*]
\item We propose \ourmethod{}—\underline{S}patial \underline{S}pectral \underline{T}emporal \underline{G}raph \underline{N}eural \underline{N}etwork—a unified graph-based framework that transforms a video into a structured graph integrating spatial, spectral, and temporal information. Unlike prior works that apply GNNs in a shallow or decoupled fashion, our approach leverages expressive graph construction to capture fine-grained manipulation artifacts and enables joint reasoning over visual, structural, and dynamic inconsistencies.
\item We introduce adaptive spectral filtering on graphs via learnable spectral functions defined over the graph Laplacian. This allows the model to discover frequency-domain anomalies in a fully data-driven manner, overcoming the limitations of fixed transforms and enhancing generalization to unseen manipulation types and diverse video domains.
\item We encode both spatial and temporal differentials as negatively weighted edges within the graph, enabling the model to effectively capture subtle spatial anomalies and temporal inconsistencies (such as jitter or abrupt transitions) that are typically difficult to detect with standard frame-wise aggregation. This formulation enhances the model’s sensitivity to temporal inconsistencies while effectively integrating spatial and spectral information.
\end{itemize}

Remarkably, \ourmethod{} achieves rich representational capacity while maintaining a highly compact architecture. \textbf{It delivers state-of-the-art detection performance while requiring up to \boldmath$42\times$ fewer parameters} than existing models. Furthermore, as demonstrated in Table~\ref{tab:efficiency}, \ourmethod{} achieves faster training and inference times as well as substantially reduced memory consumption compared to the baselines. These results collectively highlight the scalability, efficiency, and real-world practicality of our unified graph modeling framework, making it especially well-suited for deployment in resource-constrained environments.

%% file: floats/fig_intro.tex

\begin{figure}[!t]
\includegraphics[width=0.42\textwidth]{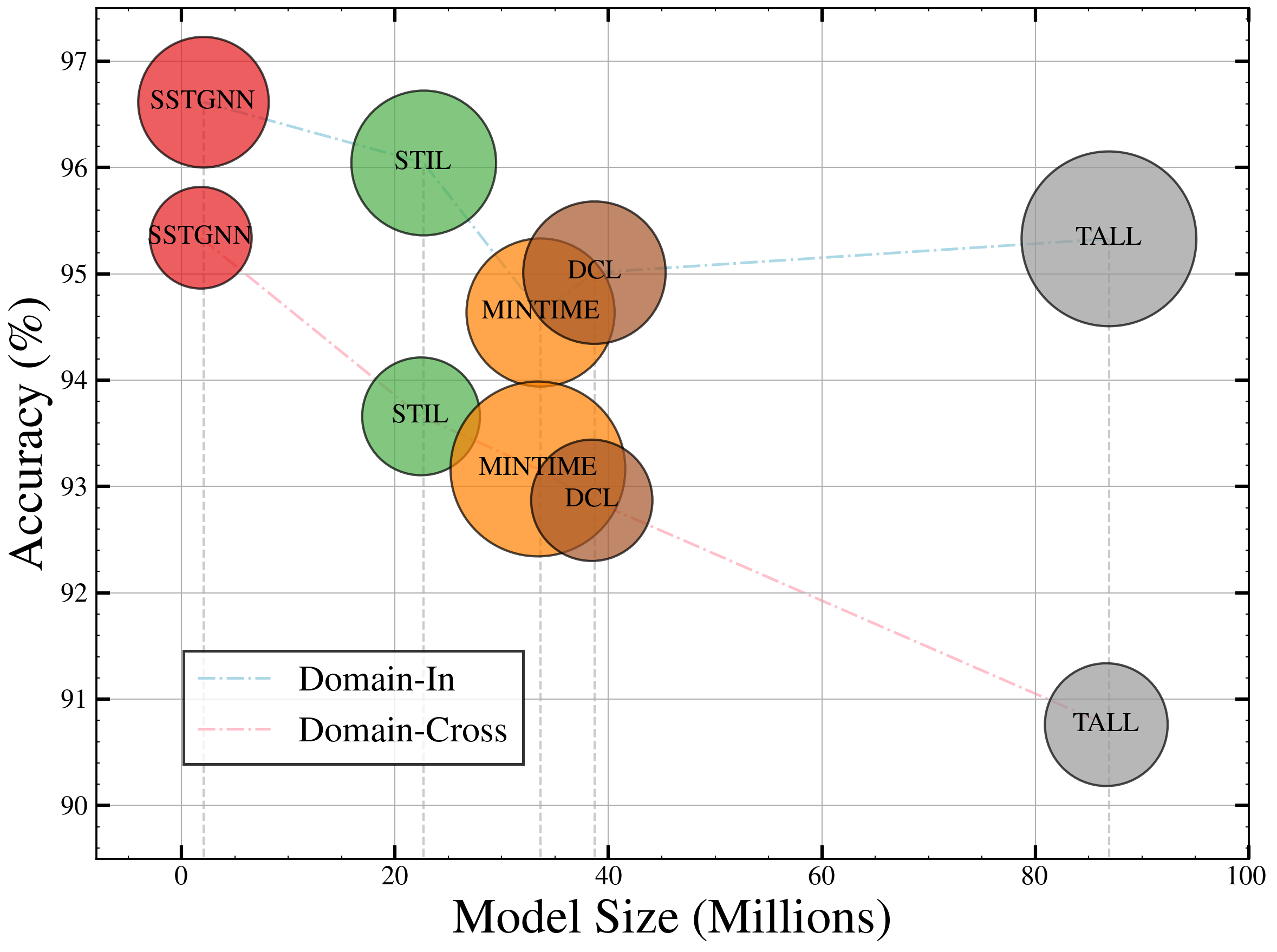}
    \vspace{-1.2em}
    \caption{Comparison of Accuracy (Y-axis), Model Size (X-axis), and Training Cost (Bubble Size) for cross-domain and in-domain deepfake video classification. Our method, \ourmethod{}, achieves superior performance over state-of-the-art baselines while requiring up to $42\times$ fewer parameters and reduced training time cost compared to the baselines.}
    \vspace{-1em}
    \label{motivation}
\end{figure}

%% file: text/related_work.tex
\section{Related Work}

\mypara{Spatial/Temporal detection methods.} A large body of prior work on deepfake detection focuses on spatial and temporal cues, typically leveraging convolutional networks. Early approaches based on architectures like XceptionNet \cite{rossler2019faceforensics} laid the groundwork by training on shared datasets in an end-to-end manner. To improve robustness, some methods target fine-grained regions, such as lips \cite{haliassos2021lips}, eye blinking \cite{jung2020deepvision}, or head pose anomalies \cite{lutz2021deepfake}, which are often poorly synthesized. To capture temporal inconsistencies, recent models incorporate temporal modules. FTCN \cite{zheng2021exploring} combines CNNs with transformers to detect short- and long-term dynamics, while attention-based architectures~\cite{liu2021nightlight} extract spatial features for fusion with temporal modules \cite{zhao2022exploring}. With the rise of Vision Transformers (ViTs), many recent models adopt hybrid CNN-ViT designs for better temporal representation learning \cite{dong2022protecting,kaddar2021hcit,wodajo2021deepfake,lei2024generative, lei2025generative, lei2025zigong}. Beyond architectural innovations, several studies address generalization to unseen forgeries \cite{ojha2023towards,wang2023altfreezing} and privacy risks such as identity leakage \cite{dong2023implicit,huang2023implicit}.


\mypara{Spectral-based detection methods.} Deepfake generation often introduces artifacts in the spectral domain due to GAN upsampling and convolutional processing. Methods such as \cite{frank2020leveraging,luo2021generalizing} detect such artifacts by analyzing the spectral distribution, offering strong generalization with low computational overhead. SPSL \cite{liu2021spatial} further combines spatial and phase spectrum features to identify subtle distortions, while AVFF \cite{oorloff2024avff} utilizes log-mel spectrograms to uncover audio-visual inconsistencies in manipulated content. Durall \textit{et al.} \cite{durall2020watch} show that GAN-generated images exhibit spectral characteristics that deviate from those of natural images. While effective, most spectral-based approaches process frequency-domain signals independently, lacking mechanisms for joint modeling across spatial, temporal, and spectral modalities.

\mypara{GNNs in Computer Vision.} GNNs are a powerful paradigm for modeling structured data~\cite{han2022vision, liu2025sigma, liao2022scara, liao2023ld2}. They have been successfully applied in tasks such as object recognition, segmentation, and video understanding \cite{kipf2016semi}. In the context of deepfake detection, recent works explore GNNs to model relational structures across space or time. Tan \textit{et al.} \cite{tan2023deepfake} use graph learning over action units to model facial region interactions, and Wang \textit{et al.} \cite{wang2023dynamic} construct spatial-frequency graphs with fixed spectral filters. However, many of these approaches treat GNNs as black-box modules, relying on fixed adjacency without explicitly designing graph topologies or integrating multi-domain signals. Anurag \textit{et al.} \cite{arnab2021unified} represent frames as nodes and build temporal edges via similarity, enabling simple yet effective video-level reasoning. In contrast, we construct a fine-grained spatiotemporal graph where each frame is a subgraph, and cross-frame connections capture long-term dynamics. This design allows unified learning of spatial, spectral, and temporal dependencies and offers improved generalizability and interpretability.

%% file: floats/fig_framework.tex
\begin{figure*}[t]
    \centering
    \includegraphics[width=1.0\textwidth]{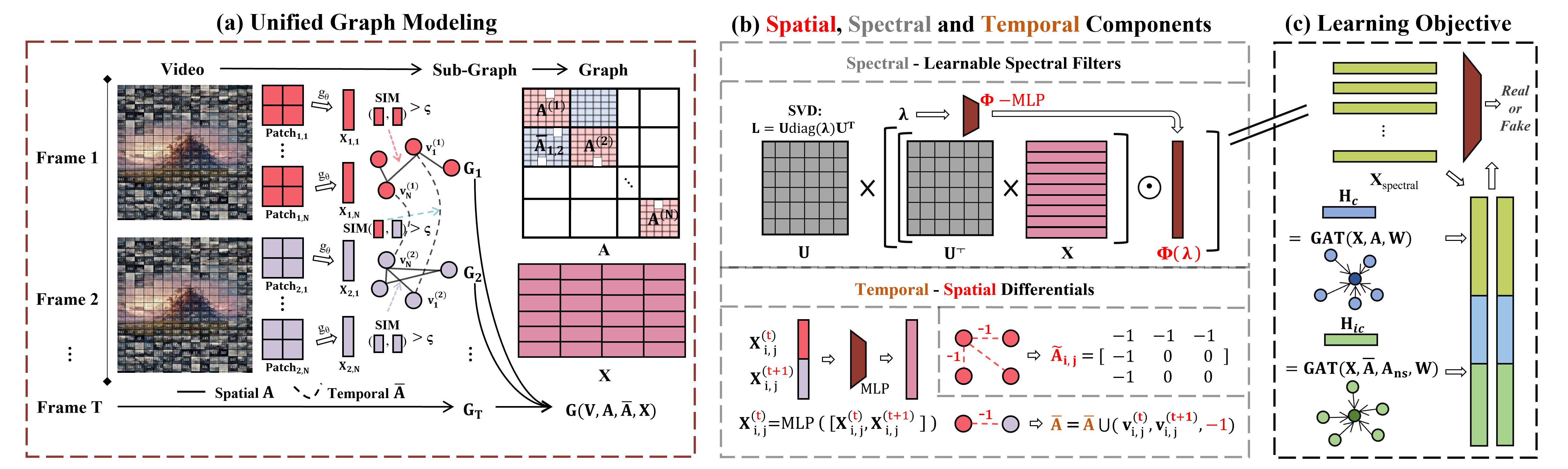}
    \vspace{-0.5cm}
\caption{Overview of our \ourmethod{} framework: (a) Each video frame is divided into patches, which are encoded as node embeddings to form intra-frame graphs {\small$\mathbf{G}_t = (\mathbf{V}_t, \mathbf{A}_t)$} based on patch similarity. Temporal edges $\mathbf{\overline{A}}$ are constructed by connecting corresponding patches across frames using both feature and structural similarity, resulting in a unified spatial-temporal graph $\mathbf{G} = (\mathbf{V}, \mathbf{A}, \mathbf{\overline{A}}, \mathbf{X})$. (b) A learnable spectral filter is applied over the graph Laplacian eigenbasis to extract frequency-domain representations. The temporal component involves concatenating embeddings and incorporating negative edges into $\mathbf{\overline{A}}$, while the spatial differential module constructs each negative sub-adjacency matrix $\mathbf{\widetilde{A}}_{i,j}$. (c) Two Graph Attention Networks (GATs) are employed to model both consistency (via positive edges) and inconsistency (via negative edges). The resulting features are concatenated and fed into a final classifier to predict real or fake videos.}
\label{fig:framework}
\end{figure*}

%% file: text/method.tex
\section{SSTGNN}
In this section, we introduce \ourmethod{}, a unified learning framework for detecting AI-generated videos. An overview can be seen in Figure~\ref{fig:framework}. By modeling each video as a graph, \ourmethod{} enables flexible and principled aggregation across spatial, spectral, and temporal domains, achieving robust detection performance with dramatically fewer parameters, faster training times, and lower memory overhead. Our key insights are summarized as follows:

\begin{itemize}[leftmargin=*]
\item \emph{Unified Graph Modeling.} Detecting AI-generated videos requires capturing both \emph{intra-frame} visual artifacts (e.g., unnatural textures or distorted facial features) and \emph{inter-frame} temporal inconsistencies. We model image patches as nodes and each frame as a subgraph, constructing spatial edges to enable fine-grained spatial representation beyond coarse frame-level modeling. To capture temporal dynamics, we add edges between corresponding patches across frames based on appearance similarity or change. This yields a spatial-temporal graph supporting flexible graph-component designs in the spatial, spectral, and temporal domains.

\item \emph{Learnable Spectral Filtering.} Building on the constructed video graph, we employ learnable spectral filters—originating from spectral GNNs~\cite{liao2025comprehensive}—to enhance frequency-aware feature extraction. Unlike traditional methods that rely on fixed or hand-crafted transforms, our approach adapts spectral filters during training, enabling more expressive and task-relevant representations.

\item \emph{Spatial-Temporal Differentials.} Inspired by recent advances such as the Neighboring Pixel Relationships (NPR) approach~\cite{tan2024rethinking}, we introduce both spatial and temporal differential modeling via negatively weighted edges. This enables the model to explicitly capture subtle spatial anomalies (e.g., local pixel artifacts) as well as temporal inconsistencies (e.g., frame-to-frame jitter or unnatural transitions), thereby improving sensitivity to the unique patterns of AI-generated content.
\end{itemize}
By integrating these complementary cues, \ourmethod{} yields a comprehensive representation for video authenticity analysis. We then provide detailed descriptions of each component in the following.

\input{floats/fig_filter_example}

\subsection{Unified Graph Modeling}
\label{sec:unified_graph}
While GNNs have significant potential for AI-generated video detection, prior works often suffer from overly coarse graph designs, treating each frame as a single node and using GNNs as black-box temporal models~\cite{wang2023dynamic}. Inspired by recent advances in patch-level image graph modeling~\cite{han2022vision, han2023vision}, we represent each frame as a subgraph of patch-level nodes. To capture temporal dynamics, we introduce temporal edges between corresponding patches across frames, forming a unified spatiotemporal graph for the video.

Formally, given a video clip $\mathcal{F} = \{\mathbf{F}^{(1)}, \mathbf{F}^{(2)}, \ldots, \mathbf{F}^{(T)}\}$ with \(T\) frames, we model each frame $\mathbf{F}^{(t)} \in \mathbb{R}^{H \times W \times C}$ at time step \(t\) as a subgraph. Each frame is divided into \(N\) image patches of size $\ell\times\ell$, and their corresponding embeddings are extracted via a patch-wise encoder \(g_\theta\), which can be any pre-trained visual backbone. This yields a node set $\mathbf{V}^{(t)} = \{v_1^{(t)}, v_2^{(t)}, \ldots, v_N^{(t)}\}$ and an embedding matrix $\mathbf{X}^{(t)} = \{x_1^{(t)}, x_2^{(t)}, \ldots, x_N^{(t)}\} \in \mathbb{R}^{N \times d}$, where \(d\) is the embedding dimension, and \(H \times W\), \(C\) are the spatial and channel dimensions of the original frame. Without loss of generality, we also denote the nodes in $\mathbf{V}^{(t)}$ as $\{v^{(t)}_{i,j}|\forall 1\leq i\leq N_H, \forall 1\leq j\leq N_W\}$ to represent its coordinates, where $N_H=\lceil{H}/{\ell}\rceil$ and $N_W=\lceil{W}/{\ell}\rceil$.

To model the intra-frame structure, we construct edges between patch nodes based on visual similarity. For each feature matrix $\mathbf{X}^{(t)}$, each embedding is normalized as
\[
\overline{X}^{(t)} = \frac{\mathbf{X}^{(t)}}{\|\mathbf{X}^{(t)}\|_2 + \epsilon}\quad \text{with } \epsilon = 10^{-4},
\]
and the intra-frame edge weight matrix is defined as
\begin{equation*}
\mathbf{A}^{(t)} = \mathrm{SIM}(\overline{X}^{(t)},\,\overline{X}^{(t)}),
\end{equation*}
where $(\mathbf{A}^{(t)})_{i,j}$ denotes the similarity between nodes $i$ and $j$ in frame $t$, and thus the edge weight between $v^{(t)}_i$ and $v^{(t)}_j$. To balance computational efficiency and significance, a threshold parameter $\tau_s$ is used to prune edges with weights below $\tau_s$. We denote the frame-level subgraph as $\mathbf{G}^{(t)}=(\mathbf{V}^{(t)},\mathbf{A}^{(t)})$. To incorporate basic temporal patterns, we construct temporal edges between every two consecutive frames. At time step $t$, we bridge the subgraphs $\mathbf{G}^{(t)}$ and $\mathbf{G}^{(t+1)}$ by adding edges for similar node pairs at the same coordinates as $\mathbf{S}^{(t)}(v^{(t)}_i, v^{(t+1)}_i):=$
\begin{align*}
\mathrm{SIM}(\mathbf{A}^{(t)}(v^{(t)}_i), \mathbf{A}_{t+1}(v^{(t+1)}_i)) + \mathrm{SIM}(\mathbf{X}^{(t)}(v^{(t)}_i), \mathbf{X}_{t+1}(v^{(t+1)}_i)),
\end{align*}
where the first term measures structural similarity using adjacency matrices and the second considers embedding vector resemblance. If the similarity surpasses threshold $\tau_t$, a temporal edge is defined:
\begin{align*}
\mathbf{S}^{(t)}(v^{(t)}_i, v^{(t+1)}_i) \geq \tau_t \implies \mathbf{\overline{A}}(v^{(t)}_i, v^{(t+1)}_i) = \mathbf{S}^{(t)}(v^{(t)}_i, v^{(t+1)}_i).
\end{align*}
By constructing $\mathbf{\overline{A}}$ across all time steps, we form the complete video graph $\mathbf{G} = (\mathbf{V}, \mathbf{A}, \mathbf{\overline{A}}, \mathbf{X})$, where the node set $\mathbf{V} = \bigcup_{t=1}^T \mathbf{V}^{(t)}$ and the adjacency matrix $\mathbf{A}$ combines spatial and temporal edges with $\mathbf{A} = \bigcup_{t=1}^T \mathbf{A}^{(t)}$. This unified graph captures both intra-frame (spatial) and inter-frame (temporal) correlations, supporting flexible and efficient spatial, spectral, and temporal reasoning in the below.

\subsection{Spatial, Spectral and Temporal Components}
\label{sec:components}
Building upon our unified graph modeling, we are able to flexibly design specialized components tailored for artifact detection.

\subsubsection{Learnable Spectral Filters}

As motivated in Section~\ref{sec:intro}, traditional frequency-based methods~\cite{wang2023dynamic} apply fixed transforms (e.g., DCT) and hand-crafted filters to video frames. A filter is a function that specifies which frequency components to emphasize or suppress, e.g., high-pass filters highlight edges, while low-pass filters smooth textures. Such operations are common in image processing for enhancing or removing specific structures.
In our case, we construct graphs from raw pixels and perform filtering in the graph spectral domain, where frequency reflects variation across the graph. This enables selective manipulation of signal components, akin to classical filters but grounded in the graph structure.

Moreover, while traditional filters inject useful inductive biases, their static and hand-crafted nature may be sub-optimal. We address this by learning \emph{arbitrary filters} from graph signals, allowing the model to flexibly extract relevant frequency patterns for improved representation. Following standard GNN practice, we use the normalized graph Laplacian:
\begin{align*}
\mathbf{L} = \mathbf{I} - \mathbf{D}^{-1/2} \mathbf{A} \mathbf{D}^{-1/2},
\end{align*}
where $\mathbf{D}$ is the degree matrix and $\mathbf{I}$ is the identity. To ensure symmetry, set $\mathbf{L} \coloneq (\mathbf{L} + \mathbf{L}^\top)/2$ and perform eigendecomposition:
\begin{align*}
\mathbf{L} = \mathbf{U} \operatorname{diag}(\boldsymbol{\lambda}) \mathbf{U}^\top,
\end{align*}
where $\boldsymbol{\lambda} \in \mathbb{R}^N$ and $\mathbf{U} \in \mathbb{R}^{N \times N}$ are the eigenvalues and eigenvectors. The learnable spectral filter is defined as
\begin{align*}
\boldsymbol{\phi}(\boldsymbol{\lambda}) = f_{\mathrm{MLP}}(\boldsymbol{\lambda}) \in \mathbb{R}^N,
\end{align*}
enabling flexible, data-driven spectral manipulation.

Node embeddings $\mathbf{X}$ are projected onto the spectral domain:
\begin{align*}
\widetilde{\mathbf{X}} = \mathbf{U}^\top \mathbf{X}.
\end{align*}
The learned spectral filter modulates these coefficients, and the result is transformed back:
\begin{align*}
\mathbf{X}_{\mathrm{spectral}} = \mathbf{U} \operatorname{diag}(\boldsymbol{\phi}(\boldsymbol{\lambda})) \widetilde{\mathbf{X}}.
\end{align*}
A global pooling operation yields aggregated spectral features:
\begin{align*}
\mathbf{Z}_{\mathrm{spectral}} = \operatorname{Pool}(\mathbf{X}_{\mathrm{spectral}}) \in \mathbb{R}^d.
\end{align*}

The intuition behind this approach is that representing raw pixels as a graph enables the model to learn arbitrary, task-specific spectral filters via parameterized functions of the eigenvalues~$\boldsymbol{\lambda}$. This flexibility allows the model to selectively emphasize or suppress frequency components most relevant to the detection task. As shown in Figure~\ref{fig:filter_example}, panels (b)–(e) illustrate the impact of applying low-pass, high-pass, band-pass, and band-rejection filters to~$\boldsymbol{\lambda}$, thereby enhancing or attenuating various image features—such as smoothing textures, highlighting edges, or isolating intermediate-scale patterns. Panel (f) (“Comb”) demonstrates that aggregating multiple frequency bands yields richer and more detailed image reconstructions. Importantly, these spectral filters are not fixed; the filter function~$\boldsymbol{\phi}(\boldsymbol{\lambda})$, parameterized by MLP layers, supports fully adaptive and data-driven filtering tailored to the downstream task. This approach overcomes the limitations of hand-crafted filters and enables more expressive and effective feature extraction.

\subsubsection{Spatial-Temporal Differentials as Negative Edges}

Recall from Section~\ref{sec:unified_graph} that we construct a unified graph using fundamental inter- and intra-frame edges by measuring similarity, providing initial \emph{consistency} among patches. However, as shown in~\cite{gu2021spatiotemporal,tan2024rethinking}, explicitly measuring spatial or temporal \emph{inconsistency} is valuable for detecting generative bias. NPR~\cite{tan2024rethinking} computes local pixel interdependence, and~\cite{gu2021spatiotemporal} exploits temporal differences across adjacent frames. Motivated by these findings, we incorporate negative edges for inconsistency message passing:

\mypara{Spatial Differential.}
The original NPR technique focuses on local interdependence induced by up-sampling. Specifically, an $\ell_0$-NPR processes each pixel grid of size $\ell_0\times\ell_0$ into local differential:
\begin{align*}
\begin{bmatrix}
w_{1,1} & \cdots & w_{1,\ell_0} \\
\vdots & \ddots & \vdots \\
w_{\ell_0,1} & \cdots & w_{\ell_0,\ell_0}
\end{bmatrix}
\rightarrow
\begin{bmatrix}
\widetilde{w}_{1,1} & \cdots & \widetilde{w}_{1,\ell_0} \\
\vdots & \ddots & \vdots \\
\widetilde{w}_{\ell_0,1} & \cdots & \widetilde{w}_{\ell_0,\ell_0}
\end{bmatrix},
\end{align*}
where $\widetilde{w}_{i,j}=w_{i,j}-w_{1,1}$, with $\ell_0=2$ shown to yield best performance~\cite{tan2024rethinking}. This approach is effective in identifying fake regions.

We generalize this process to node embeddings by adding negative edges in our graph. Specifically for all $1 \leq i, j \leq \lceil N/\ell_0 \rceil$, $1 \leq t \leq T$, for node $v^{(t)}_{i_0,j_0}$ with $i_0 = i\ell_0,\, j_0 = j\ell_0$, we then construct a sub-adjacency matrix $\mathbf{\widetilde{A}}_{i,j} \in \mathbb{R}^{\ell_0 \times \ell_0}$ correspondingly:
\begin{align*}
    \mathrm{diag}(\mathbf{\widetilde{A}}_{i,j}) = \mathbf{1};\quad 
    \mathbf{\widetilde{A}}_{i,j}(v^{(t)}_{i_0, j_0}, :) = -\mathbf{1};\quad  
    \mathbf{\widetilde{A}}_{i,j}(:, v^{(t)}_{i_0}) = -\mathbf{1}.
\end{align*}
This sub-adjacency matrix is later used for message passing in graph learning (see Section~\ref{sec:graph_learning}). Intuitively, our approach captures local interdependence among nodes, applicable to pixels or higher-level embeddings. As proven in Theorem~\ref{lem:NPR}, with appropriate settings, $\ell_0$-NPR is a special case of our formulation after message passing.

\begin{theorem}[Proof in Appendix~\ref{appendix:proof_NPR}]
\label{lem:NPR}
An $\ell_0$-NPR is equivalent to our spatial differential module when a small patch size is used and traditional message passing is adopted; specifically, by setting patch size $\ell=1$ and using SGC-aggregation~\cite{wu2019simplifying}.
\end{theorem}

This theorem demonstrates that our spatial differential module not only subsumes the $\ell_0$-NPR as a special case but also extends its capabilities, highlighting the effectiveness and versatility of our framework.

\mypara{Temporal Differential.}
To capture temporal inconsistency, we enhance each node’s representation by integrating information from adjacent frames. Intuitively, AI-generated videos may show subtle inconsistencies or unnatural transitions that are missed by spatial features alone. By explicitly modeling temporal relationships, we improve artifact detection. For node $v^{(t)}_{i,j}$ at $1 \leq t < T$, with raw embedding $\mathbf{X}^{(t)}_{i,j}$, we concatenate features from the next frame and apply an MLP:
\begin{align*}
\mathbf{X}^{(t)}_{i,j} = \mathrm{MLP}\big(\mathrm{Concat}[\mathbf{X}_{i,j}^{(t)},\, \mathbf{X}_{i,j}^{(t+1)}]\big).
\end{align*}
This enables implicit integration of temporal dynamics, facilitating detection of subtle frame-to-frame changes. We further introduce negative edges into the inter-frame adjacency matrix:
\begin{align*}
\mathbf{\overline{A}}(v^{(t)}_{i,j}, v^{(t+1)}_{i,j}) = -1,
\end{align*}
for all $1 \leq i \leq N_H$, $1 \leq j \leq N_C$, and $1 \leq t < T$. These negative edges penalize abrupt changes in node features across consecutive time steps, thereby guiding the model’s attention to regions exhibiting pronounced temporal discrepancies. As demonstrated in Section~\ref{sec:ablation}, incorporating these negative temporal edges substantially enhances the overall detection performance.

By combining implicit feature integration with explicit relational modeling, our temporal differential module proves highly effective in revealing frame-to-frame inconsistencies that are characteristic of synthesized or manipulated videos, thereby significantly improving detection accuracy as demonstrated later in Section~\ref{sec:ablation}.

\subsubsection{Conventional Learning Model}
\label{sec:graph_learning}

With all edges constructed, we apply the typical Graph Attention Network (GAT)~\cite{velivckovic2018graph} to our video graph due to its effectiveness in message passing. A single-layer GAT $(\mathbf{X},\mathbf{A},\mathbf{W})$ applies a shared linear transformation $\mathbf{W}\in\mathbb{R}^{d\times d}$ to each node feature $\mathbf{x}_i \in \mathbb{R}^d$:
$
\mathbf{h}_i = \mathbf{W}\mathbf{x}_i.
$
For each edge $(i, j)$, it computes an unnormalized attention score:
\[
e_{i,j} = \mathrm{LeakyReLU}\bigl(\mathbf{a}^\top \mathrm{Concat}[\mathbf{h}_i \Vert \ \mathbf{h}_j]\bigr),
\]
where $\mathbf{a}$ is a learnable vector and $\Vert$ denotes concatenation. Attention scores are masked by the adjacency matrix $\mathbf{A}$ and normalized:
\[
\alpha_{i,j} = \frac{\exp(e_{i,j})}{\sum_{k \in \mathcal{N}(i)} \exp(e_{ik})}.
\]
Each node then aggregates neighbors’ features with a nonlinear activation, denoted as $\sigma$ in the below:
\[
\mathbf{h}_i' = \sigma\left(\sum_{j \in \mathcal{N}(i)} \alpha_{i,j} \mathbf{h}_j\right).
\]
We then employ two GAT layers to capture both the consistency and inconsistency information:
\begin{align*}
\mathbf{H}_c  = \mathrm{GAT}(\mathbf{X}, \mathbf{A}, \mathbf{W}) \text{ and } \mathbf{H}_{ic} &= \mathrm{GAT}(\mathbf{X},\, \mathbf{\overline{A}} \cup \mathbf{\widetilde{A}},\, \mathbf{W}),
\end{align*}
where $\mathbf{A}$ encodes standard edges, $\mathbf{\overline{A}}$ complementary connections, and $\mathbf{\widetilde{A}}$ negative spatial edges. Final node representations are then concatenated and processed by an MLP layer for spatial embedding:
\[
\mathbf{Z}_{\text{spatial}} = \mathrm{MLP}\left(\mathrm{Concat}\left[\mathbf{H}_c\Vert \mathbf{H}_{ic}\right]\right).
\]
This spatial representation will then be combined with the spectral feature to support binary classification of video authenticity.

\subsection{Learning Objective}
Through the above steps, we have established a unified graph modeling and learning framework tailored for AI-generated video detection. In the final stage, all extracted representations are integrated for classification. Specifically, we concatenate learned embeddings from the spatial, spectral, and temporal domains:
\[
\mathbf{Z} = [\mathbf{Z}_{\text{spatial}}\Vert \mathbf{Z}_{\text{spectral}}].
\]
This overall representation is passed through an MLP to produce classification logits:
\[
\mathbf{y} = \mathrm{MLP}(\mathbf{Z}) \in \mathbb{R}^{N \times 2},
\]
where label 0 denotes authentic videos and label 1 indicates AI-generated content. The model is trained end-to-end using cross-entropy loss to optimize classification. The full procedure of \ourmethod{} is outlined in Algorithm~\ref{alg:SSTGNN} of Section~\ref{sec:algo}.

%% file: floats/fig_filter_example.tex
\begin{figure}[!t]
\centering
\includegraphics[width=0.48\textwidth]{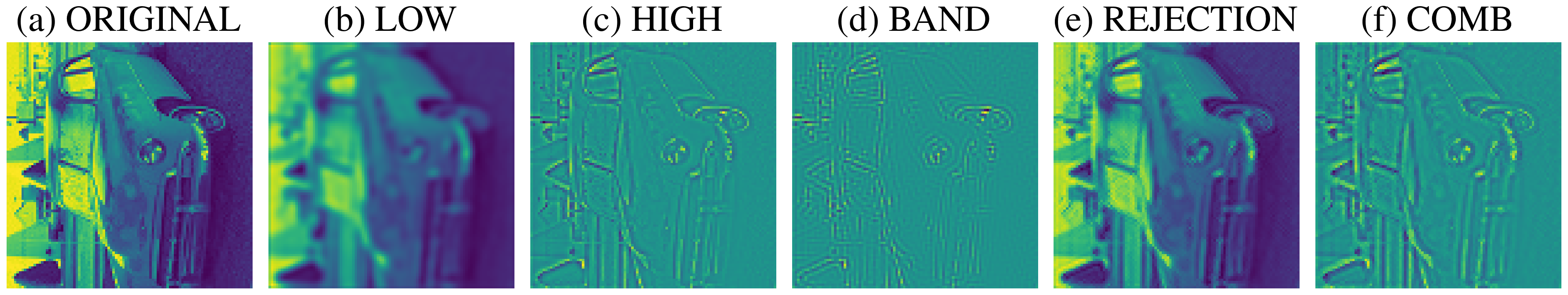}
\vspace{-3.ex}
\caption{Illustration of an input image after our graph spectral filtering with specified function.}
\label{fig:filter_example}
\vspace{-1.ex}
\end{figure}

%% file: text/experiments.tex
\section{Experiments}
In this section, we explore the following research questions (RQs):

\begin{itemize}[leftmargin=*]
\item \textbf{RQ1: In-domain Detection Performance.} How effective is \ourmethod{} when trained and tested on the same deepfake videos?
\item \textbf{RQ2: Cross-domain Generalization.} How well does \ourmethod{} generalize across different video generation domains, specifically (i) cross-method generalization, where \ourmethod{} is trained on one video generation method and tested on another, and (ii) cross-dataset generalization, evaluating \ourmethod{}'s performance across multiple datasets in one-to-many and many-to-many settings?
\item \textbf{RQ3: Resource Allocation (Memory, Time, and Number of Parameters).} How lightweight is \ourmethod{} in terms of memory consumption, training time, and model memory footprint?
\item \textbf{RQ4: Ablation Study.} What are the individual contributions of \ourmethod{}'s components, patch size, and pruning threshold?
\item \textbf{RQ5: Interpretability.} How does SSTGNN capture key artifacts with learned spectral functions and gradient-based heatmaps.
\end{itemize}

\mypara{Baselines.} We compare \ourmethod{} against 10 state-of-the-art baseline methods, which were selected based on their relevance and proven effectiveness in deepfake detection tasks. These baselines include various approaches, ranging from spatial-temporal analysis to graph-based and attention-based methods, as summarized in Table \ref{baselines}. For open-sourced baselines, we run their official implementations using the hyper-parameters specified in their respective papers. For baselines without publicly available code, we directly report the results from the original publications under consistent settings.

\input{floats/tab_descrip_baselines}

\mypara{Datasets.} We evaluate \ourmethod{} and the 10 baselines on a comprehensive range of datasets, as detailed in Table \ref{datasets} in Appendix \ref{app1}. These datasets are widely used in the deepfake detection community. The FF++ dataset consists of four subsets, each generated with a distinct face manipulation technique: DeepFakes (DF), Face2Face (F2F), FaceSwap (FS), and NeuralTextures (NT). These subsets facilitate a systematic evaluation across different types of facial forgeries.

\mypara{Experimental Setup.} To address \textbf{RQ1}, we perform training and testing on the same dataset using six datasets: DF, F2F, FS, NT, CD (v1 and v2), and Wild-DF. For \textbf{RQ2}, we employ a three-pronged approach: (i) fine-grained cross-subset testing within the four FF++ subsets, training each model on one subset and testing it on the other three; (ii) training a model on the full FF++ dataset and evaluating it on three external datasets: Wild-DF, CD-v2, and DFDC; and (iii) testing \ourmethod{} in a 3-to-3 setting, where training data consists of SEINE, SVD, and Pika, and test data includes Crafter, Gen2, and Lavie. For \textbf{RQ3}, all models are trained on the following datasets: DF, F2F, FS, NT, CD-v1, Wild-DF, SEINE, and SVD. We record both the peak and average GPU memory usage, as well as training time, to assess resource efficiency. Ablation studies for \textbf{RQ4} are conducted exclusively on the NT dataset, enabling us to evaluate the impact of individual components on overall performance.

\mypara{Model Configuration.} Our \ourmethod{} uses a modified ResNet-50 as the backbone, where we retain the network up to layer-2, followed by a global average pooling layer to generate compact feature representations. The model is trained using the Adam optimizer with a learning rate of 0.0001, a batch size of 16, and a total of 100 epochs. Performance is evaluated using two metrics: Area Under the ROC Curve (AUC) and Accuracy (Acc). All experiments are conducted on a machine with three NVIDIA RTX A5000 GPU (24GB) and two Intel Xeon Gold 6238R CPUs (2.20GHz).

\subsection{\textbf{RQ1}: In-domain Detection Performance}
\input{floats/tab_in_domain}
Table~\ref{indomain1} compares our method with seven state-of-the-art deepfake detection approaches across seven datasets. Our method consistently ranks as a top performer, achieving the highest average AUC (99.51\%) and the best overall rank (1.14) across all datasets, highlighting its superior in-domain performance and strong generalization. Specifically, our model achieves the best AUC on FS (99.97\%), DF (99.99\%), CD-V1 (99.96\%), CD-V2 (99.86\%), and Wild-DF (94.36\%), while also securing second-best AUC on NT (98.76\%) and F2F (99.67\%). This demonstrates our model’s robustness across both canonical manipulations and more diverse, challenging settings (e.g., CD-V1/V2 and Wild-DF).
Compared to recent methods like STIL and DCL, our approach not only matches or exceeds their performance on most datasets but also maintains full evaluation coverage, unlike some baselines with missing results. Moreover, earlier methods such as FaceXray and MultiAtt show significantly lower performance, especially on unconstrained datasets. Overall, the results demonstrate that \ourmethod{} offers a compelling balance of accuracy and reliability across various deepfake detection scenarios.

\subsection{RQ2: Cross-domain Generalization}
We assess cross-domain generalization under comprehensive experimental settings and provide explainable analyses to showcase the superior generalizability of \ourmethod{}.

\mypara{1-to-3 setting.} Table~\ref{cross1} presents the cross-method generalization results for various deepfake detection approaches, where each model is trained on a specific forgery type and tested on other types within the FF++ dataset. This evaluation highlights each method’s ability to generalize across different manipulation techniques. Our method consistently outperforms competing approaches across most training settings. When trained on DF, our method achieves the highest average accuracy (73.82\%), and obtains the best results on both NT (83.42\%) and F2F (79.34\%). Similarly, with F2F as the training type, our model again ranks first (72.16\%), significantly outperforming others on NT (76.67\%).  
Under FS training, RECCE achieves the highest average (75.84\%), closely followed by our method (72.56\%), with strong results on F2F (66.48\%) and DF (69.99\%). In the NT-trained setting, our method attains a remarkable 84.28\% average accuracy—the best overall—demonstrating superior generalization to unseen videos such as F2F (95.19\%).

\input{floats/tab_cross_3-3}
\input{floats/tab_cross_FF}

\mypara{3-to-3 setting.}  
To further evaluate robustness in realistic scenarios, we adopt a 3-to-3 setting, where the model is trained on SEINE, SVD, and Pika, and tested on Crafter, Gen2, and Lavie (Table~\ref{cross5}). This setup simulates diverse generative patterns in both training and testing. Our method consistently achieves strong and balanced performance across all test datasets. On Crafter, it attains the highest accuracy (95.83\%) and AUC (99.62\%), outperforming STIL and DCL. On Gen2, it achieves the top AUC (98.96\%) and competitive accuracy (93.06\%). While baselines like STIL and DCL may perform well on individual domains, they tend to experience notable drops in performance when evaluated across other domains.

\mypara{Remark.} Under comprehensive cross-domain elevations, we highlight \ourmethod{} for its out-standing generation ability, benefiting from our unified graph representation design.

\mypara{Generalization Analyses.} To assess the generalization capabilities of our method, we conduct t-SNE visualization and Principal Component Analysis (PCA) on both \ourmethod{} and the runner-up method, STIL, using the FF++ dataset. Figure~\ref{vaoss} showcases the t-SNE embeddings of feature representations. In panel (a), the features from STIL reveal a significant clustering of seen fake samples, forming a dense cluster dominated by fake samples, as highlighted by the red ellipse. However, other samples—particularly unseen fakes—are more diffusely distributed and not distinctly separated from real samples. This suggests that STIL tends to overfit to the known fake patterns, restricting its ability to generalize effectively to unseen manipulations. In contrast, panel (b) illustrates that \ourmethod{} produces a more balanced feature space, where seen fake, unseen fake, and real samples are better separated, albeit with some natural overlap. This indicates that \ourmethod{} excels at capturing subtle forgery characteristics, thus improving its generalization across diverse fake video types. To quantitatively analyze the information content within the learned feature representations, we apply Principal Component Analysis (PCA), examining the explained variance across principal components, as shown in Figure~\ref{vaoss2}. For STIL, over 90\% of the variance is concentrated within the top three components, suggesting a low-rank structure that limits its expressiveness and generalization potential. In contrast, \ourmethod{} distributes the variance more evenly across the first seven to eight components, indicating a much richer and more expressive feature space. This broader distribution of variance enables \ourmethod{} to capture diverse and previously unseen forgery patterns, reducing the risk of overfitting to shallow, dataset-specific features. Such a structure enhances the model’s ability to generalize robustly on various deepfake detection tasks.

\input{floats/fig_tsne}
\input{floats/fig_pca}

\subsection{RQ3: Resource Allocation}

To evaluate the practical efficiency of our framework, we benchmark \ourmethod{} against strong baselines on FF++ in terms of training time, inference latency, and GPU memory usage. As reported in Table~\ref{tab:efficiency}, \ourmethod{} consistently demonstrates superior efficiency across all evaluated metrics. In particular, it achieves the lowest training time (4698–4746 seconds), yielding a 1.27$\times$ reduction compared to the second-best method. During inference, SSTGNN processes each sample in only 10–11 milliseconds, enabling real-time or low-latency deployment and providing speedups of 1.20$\times$ and 1.09$\times$ over competing approaches. Moreover, SSTGNN maintains the smallest GPU memory footprint (8540–8546~MB), corresponding to over 2.20$\times$ memory savings. Taken together, these results underscore the strong computational efficiency and scalability of our framework: SSTGNN achieves state-of-the-art detection performance while significantly reducing training cost, inference latency, and memory consumption, making it particularly suitable for large-scale experiments as well as real-world, resource-constrained scenarios.

\input{floats/tab_resource}

\subsection{RQ4: Ablation Study}
\label{sec:ablation}
To better understand the contribution of each component and the impact of hyper-parameters, we conduct an ablation study divided into two parts. The first part isolates the effect of three core components, i.e., spatial features, spectral features, and the NPR module. The second part explores how patch size and the threshold $\tau$ influence performance, offering insight into the robustness and sensitivity of the model under different configurations.

\input{floats/tab_ablation}

\mypara{Spatial, Spectral and Temporal Components.} 
Table~\ref{abla1} presents an ablation study quantifying the individual and combined contributions of the spatial, spectral, and temporal components in our framework. Removing any one component leads to a noticeable drop in both accuracy and AUC, confirming the complementary benefits of each design. Notably, removing the spectral module degrades the performance at the most, with nearly $2\%$ and $1\%$ drops in Accuracy and AUC respectively. This greatly demonstrates the effectiveness of our learnable filter designs. Meanwhile, the negative edges involved also help the model to enhance its performance. These findings validate that integrating spatial, spectral, and temporal cues is critical in deepfake video analysis.

\mypara{Patch size and graph thresholds.} Figure~\ref{fig:ablation} examines how different patch sizes and threshold $\tau$ (both $\tau_s$ and $\tau_t$) values affect the model's performance. Among patch sizes, 32×32 achieves the best balance between detail and contextual information, reaching 95.63\% accuracy and 98.76\% AUC. Interestingly, increasing to 56×56 further improves AUC slightly (99.03\%) but offers no significant gain in accuracy, suggesting a saturation point. As for the threshold $\tau$, which controls decision confidence, the best performance is observed at $\tau=0.6$, aligning with the main setting of model. Both too high ($\tau=0.8$) and too low ($\tau=0.2$) degrade accuracy, indicating that overly strict or overly permissive thresholds can harm decision reliability. These results demonstrate that the model is relatively robust but still benefits from easy-to-implement tuning.
\input{floats/fig_ablation}


\subsection{RQ5: Interpretability}
\input{floats/fig_visualization}
To investigate the interpretability of SSTGNN, we employ two complementary visualization techniques on raw video inputs: learned graph spectral filtering results and patch-level heatmaps. As discussed in Section~\ref{sec:components}, the proposed spectral filter functions operate on graph signals and adaptively extract task-relevant frequency patterns, while gradient-based heatmaps are commonly used to probe model interpretability by highlighting input regions that contribute most to the prediction. In Figure~\ref{fig:visualization}, we present a representative video sample from the \textit{NT} dataset, showing the original frames together with their corresponding spectral-filtered representations and heatmaps produced by SSTGNN after training.

While the original frames (1--5) appear visually coherent, with stable facial appearance and smooth motion, the spectral representations reveal informative patch-level structural variations that are not apparent in the RGB domain. In particular, spectral responses exhibit non-uniform intensity across facial regions and vary across consecutive frames despite minimal visual differences. For example, in Frame~3, the spectral activation becomes more pronounced in the central facial area compared to adjacent frames, indicating sensitivity to fine-grained frequency variations across neighboring patches rather than surface-level appearance cues. In subsequent frames, spectral responses shift subtly toward the upper facial and forehead regions while the RGB content remains stable, suggesting that SSTGNN captures evolving frequency-domain inconsistencies over time. The patch-level heatmaps further clarify how these spectral cues influence the model’s attention. Across frames, higher responses consistently concentrate on localized facial regions, particularly around the central face and forehead, rather than spreading uniformly across the entire face. While the overall spatial pattern remains stable, its intensity varies gradually over time, reflecting sensitivity to temporal changes under smooth motion. Boundary regions between the face and background also receive elevated responses in some frames, highlighting areas where generative processes often introduce subtle blending or texture inconsistencies. Overall, these visualizations indicate that SSTGNN bases its predictions on interpretable, frame-specific spectral and structural cues rather than identity- or dataset-specific appearance patterns.

%% file: floats/tab_descrip_baselines.tex

%









\begin{table}[!t]
\centering
\Huge
\caption{Description and comparison of baselines. We highlight \ourmethod{} for its unified capture of spatial, spectral and temporal designs, and ultra-light model sizes.}
\vspace{-0.2cm}
\renewcommand{\arraystretch}{1.6}
\resizebox{0.48\textwidth}{!}{
\begin{tabular}{l|ccc}
\toprule
\textbf{Methods} 
& \textbf{Publication} 
& \textbf{Features} 
& \textbf{Model Size} \\ 
\midrule
MINTIME \cite{coccomini2024mintime}
& IEEE TIFS 2024
& Spatial, Temporal
& 33.64m \\ 

TNEM \cite {she2024using}
& IEEE TIFS 2024
& Spatial, Temporal
& - \\ 

TALL \cite{coccomini2024mintime}
& ICCV 2023
& Spatial, Temporal
& 86.91m \\ 

DCL \cite{sun2022dual}
& AAAI 2022
& Spatial
& 38.69m \\ 

SPLAFS \cite{chen2022learning}
& IEEE TCSVT 2022
& Spatial, Temporal
& - \\ 

RECCE \cite{cao2022end}
& CVPR 2022
& Spatial
& - \\ 

STIL \cite{gu2021spatiotemporal}
& ACM MM 2021 
& Spatial, Temporal
& 22.69m \\ 

MultiAtt \cite{zhao2021multi}
& CVPR 2021
& Spatial
& - \\ 

FaceXray \cite{li2020face}
& CVPR  2020
& Spatial
& - \\

XceptionNet \cite{chollet2017xception}
& CVPR 2017
& Spatial 
& 22.90m \\ 

\midrule
\textbf{\ourmethod{}} & This paper & \color{red}{Spatial, Spectral, Temporal} & {\color{red}\textbf{2.05m}} \\
\bottomrule
\end{tabular}}
\label{baselines}
\end{table}

%% file: floats/tab_in_domain.tex
\begin{table}[!t]
\Huge
\centering
\caption{In-domain performance comparison (AUC). {\color{red}\textbf{Red}} marks best performance and {\color{blue}\textbf{blue}} marks indicate the runner-up. \textbf{Rank} denotes the average ranking of all methods.}
\vspace{-0.2cm}
\resizebox{0.47\textwidth}{!}{
\renewcommand{\arraystretch}{1.8}
\begin{tabular}{l|ccccccc|c|c}
\toprule
\textbf{Methods} & \textbf{NT} & \textbf{FS} & \textbf{F2F} & \textbf{DF} & \textbf{CDV1} & \textbf{CDV2} & \textbf{Wild} & \textbf{Avg} & \textbf{Rank} \\
\midrule
MINTIME & 92.15 & 98.6 & 98.79 & 99.3  & 97.52 & 95.48  & 81.46 & 94.64 & 9.00 \\
TALL & 94.45 & 97.89 & 97.69 & 99.05 & 96.79 & 94.97  & 86.11 & 95.50 & 8.00 \\
MultiAtt & 93.44 & 98.87 & 97.96 & 98.84 & 99.42 & 97.86 & 90.74 & 96.16 & 7.00 \\
RECCE & 93.63 & 98.82 & 98.06 & 99.65 & 99.94 & 96.81 & 92.02 & 96.13 & 6.00 \\
SPLAFS & 97.71 & 99.27 & 99.05 & 99.55 & 98.72 & 97.32 & 91.89 & 97.50 & 6.29 \\
FaceXray & {\color{red}\textbf{98.93}} & 99.20 & 99.06 & 99.17 & 98.35 & 97.89 & 92.11 & 97.10 & 5.29 \\
DCL & 97.69 & 99.90 & 99.28 & 99.76 & {\color{red}\textbf{99.96}} & 99.77 & 88.96 & 97.76 & 4.29 \\
TNEM & 98.06 & 99.90 & 99.33 & 99.98 & 99.59 & 98.51 & {\color{blue}\textbf{92.63}} & 98.14 & 4.14 \\
STIL & 98.72 & {\color{blue}\textbf{99.91}} & {\color{red}\textbf{99.86}} & {\color{red}\textbf{99.99}} & 99.90 & {\color{red}\textbf{99.86}} & 92.36 & {\color{blue}\textbf{98.80}} & {\color{blue}\textbf{2.86}} \\
\midrule
\textbf{\ourmethod{}} & {\color{blue}\textbf{98.76}} & {\color{red}\textbf{99.97}} & {\color{blue}\textbf{99.67}} & {\color{red}\textbf{99.99}} & {\color{red}\textbf{99.96}} & {\color{red}\textbf{99.86}} & {\color{red}\textbf{94.36}} & {\color{red}\textbf{99.51}} & {\color{red}\textbf{1.14}} \\
\bottomrule
\end{tabular}}
\label{indomain1}
\end{table}

%% file: floats/tab_cross_3-3.tex
\begin{table}[!t]
\small
\centering
\caption{Cross-domain: 3-to-3 setting from training on \textit{SVD}, \textit{Seine}  and \textit{Pika} , and testing on \textit{Crafter}, \textit{Gen2} and \textit{Lavie}.}
\vspace{-0.2cm}
\setlength{\tabcolsep}{6pt}
\resizebox{0.46\textwidth}{!}{
\renewcommand{\arraystretch}{1.1}
\begin{tabular}{l|cc|cc|cc}
\toprule
\multirow{2}{*}{\textbf{Methods}} & 
\multicolumn{2}{c|}{\textbf{Crafter}} & 
\multicolumn{2}{c|}{\textbf{Gen2}} & 
\multicolumn{2}{c}{\textbf{Lavie}}  \\
\cmidrule(lr){2-3} \cmidrule(lr){4-5} \cmidrule(lr){6-7}
 & Acc & AUC & Acc & AUC & Acc & AUC \\
\midrule
STIL      & {\color{blue}\textbf{95.74}} & 97.12  & {\color{red}\textbf{97.39}} & 98.23 & 87.40 & 90.80 \\
DCL       & 93.07 & {\color{blue}\textbf{98.78}}  & {\color{blue}\textbf{93.63}} & {\color{blue}\textbf{98.88}} & {\color{red}\textbf{90.97}} & {\color{blue}\textbf{93.90}} \\
MINTIME & 92.77 & 96.57  & 92.37 & 97.56 & 85.00 & 87.84 \\
TALL      & 91.29 & 97.64  & 92.19 & 98.67 & 84.81 & 87.45  \\ \midrule
\textbf{\ourmethod{}}    & {\color{red}\textbf{95.83}} & {\color{red}\textbf{99.62}}  & 93.06 & {\color{red}\textbf{98.96}} & {\color{blue}\textbf{89.72}} & {\color{red}\textbf{94.48}}  \\
\bottomrule
\end{tabular}}
\vspace{-0.1cm}
\label{cross5}
\end{table}

%% file: floats/tab_cross_FF.tex
\begin{table}[!t]
\small
\centering
\caption{Cross-domain: 1-to-3 setting within four FF+ subsets. "-" denotes that the method is trained on that subset.}
\vspace{-0.2cm}
\resizebox{0.46\textwidth}{!}{
\setlength{\tabcolsep}{8pt}
\renewcommand{\arraystretch}{1.0}
\begin{tabular}{l|c|c|c|c|c}
\toprule
\textbf{Methods} & \textbf{DF} & \textbf{F2F} & \textbf{FS} & \textbf{NT}  & \textbf{Avg} \\
\midrule
MultiAtt 
  & - & 66.41 & {\color{blue}\textbf{67.33}} & 66.01  & 66.58 \\
STIL  & - & 65.70 & 40.74 & 75.10  & 60.51 \\
 DCL & - & {\color{blue}\textbf{77.13}} & 61.01 & {\color{blue}\textbf{77.01}}  & 71.00 \\
 RECCE  & - & 70.66 & {\color{red}\textbf{74.29}} & 67.34  & {\color{blue}\textbf{72.49}} \\
 \textbf{\ourmethod{}}  & - & {\color{red}\textbf{79.34}} & 58.71 & {\color{red}\textbf{83.42}} & {\color{red}\textbf{73.82}} \\
\midrule
MultiAtt  & 73.04 & - & {\color{red}\textbf{65.10}} & 71.88  & 67.00 \\
 STIL & 76.63 & - & 51.24 & {\color{blue}\textbf{73.30}}  & 67.39 \\
 DCL & {\color{red}\textbf{91.91}} & - & 59.58 & 66.72  & 67.25 \\
 RECCE  & 75.99 & - & 62.71 & 72.32  & {\color{blue}\textbf{69.27}} \\ 
 \textbf{\ourmethod{}}  & {\color{blue}\textbf{82.38}} & - & {\color{blue}\textbf{62.90}} & {\color{red}\textbf{76.67}}  & {\color{red}\textbf{72.16}} \\
\midrule
 MultiAtt & {\color{blue}\textbf{82.33}} & 66.11 & - & {\color{red}\textbf{63.31}}  & {\color{blue}\textbf{74.40}} \\
 STIL  & 38.20 & 49.60 & - & 49.45  & 56.83 \\
 DCL  & 74.80 & {\color{red}\textbf{69.75}} & - & 52.60  & 65.19 \\
 RECCE & {\color{red}\textbf{82.39}} & 64.44 & - & {\color{blue}\textbf{56.70}}  & {\color{red}\textbf{75.84}} \\
 \textbf{\ourmethod{}} & 69.99 & {\color{blue}\textbf{66.48}} & - & 55.79  & 72.56 \\
\midrule
 MultiAtt & 74.56 & 80.61 & {\color{blue}\textbf{66.07}} & -  & 72.88 \\
 STIL & {\color{blue}\textbf{93.28}} & {\color{blue}\textbf{89.17}} & 51.24 & -  & 73.48 \\
 DCL & 91.23 & 52.13 & {\color{red}\textbf{79.31}} & -  & 73.61 \\
 RECCE & 78.83 & 80.89 & 63.70 & -  & {\color{blue}\textbf{74.27}} \\
 \textbf{\ourmethod{}} & {\color{red}\textbf{93.58}} & {\color{red}\textbf{95.19}} & 64.07 & -  & {\color{red}\textbf{84.28}} \\
\bottomrule
\end{tabular}}
\label{cross1}
\end{table}

%% file: floats/fig_tsne.tex
\begin{figure}[!t]
    \centering
    \subcaptionbox{STIL\label{fig:tsne:stil}}[0.49\linewidth]
    {\includegraphics[width=\linewidth]{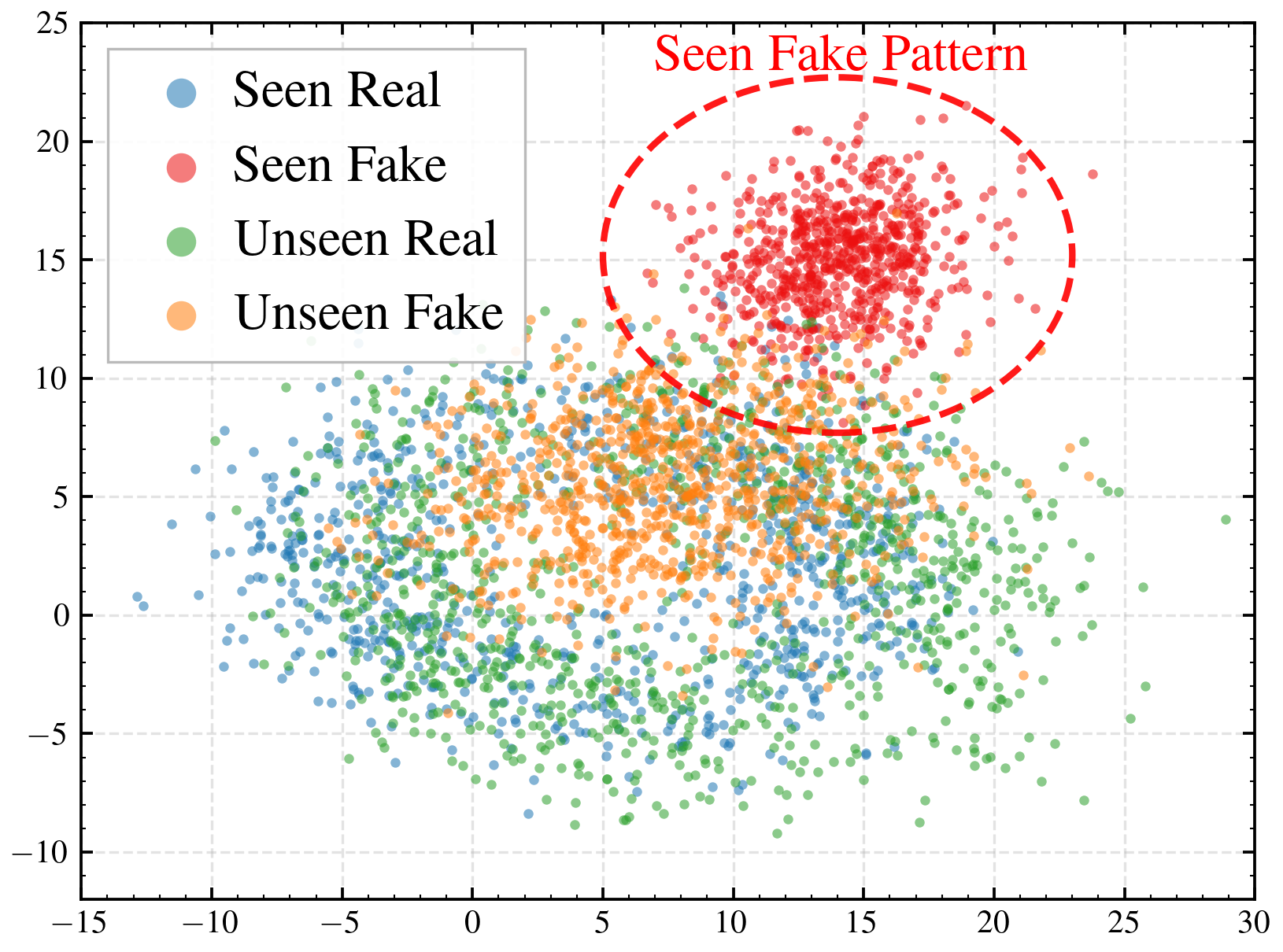}}
    \hfil
    \subcaptionbox{SSTGNN(ours)\label{fig:tsne:sstgnn}}[0.49\linewidth]
    {\includegraphics[width=\linewidth]{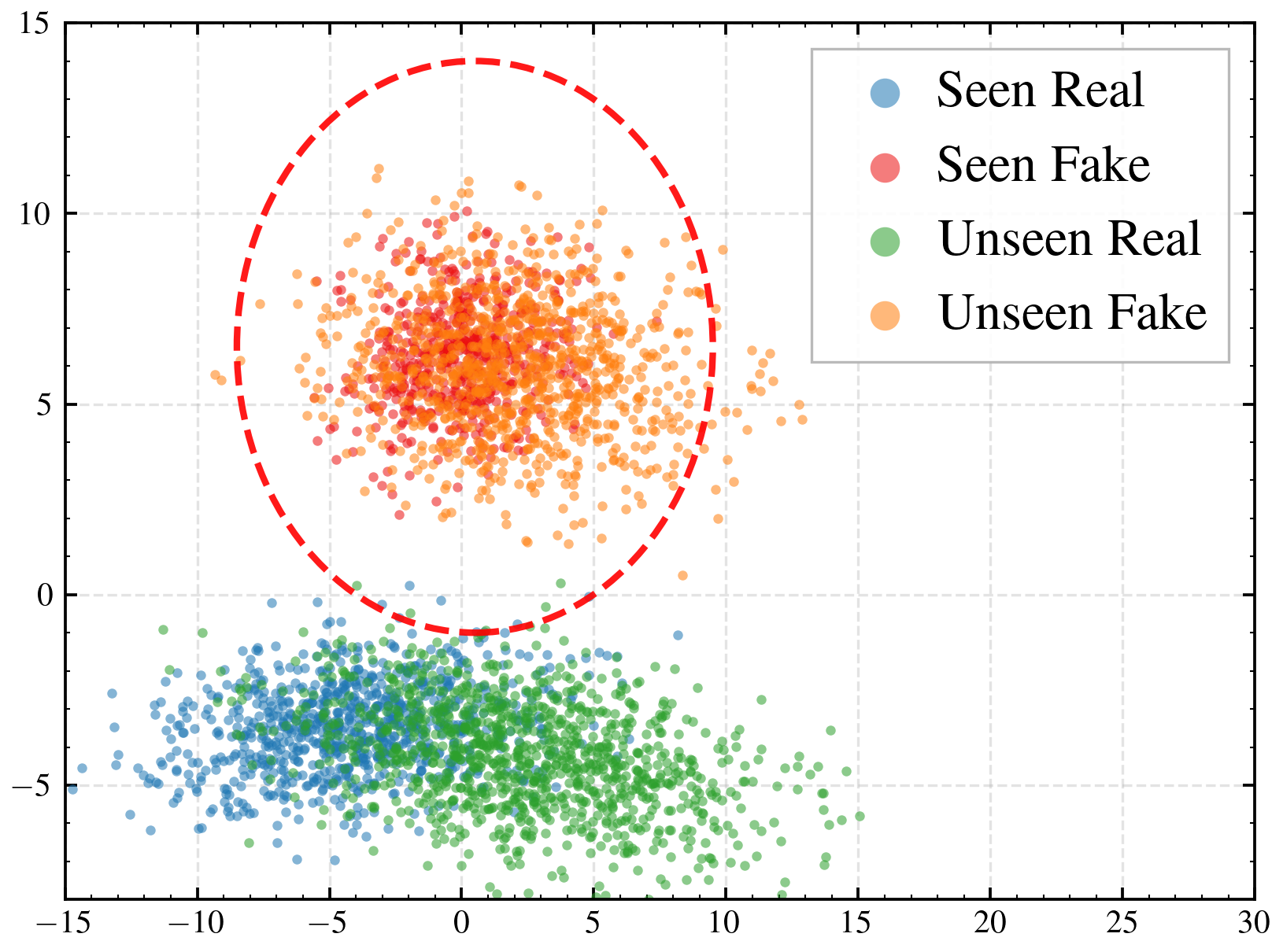}}
    \vspace{-0.3cm}
    \caption{t-SNE visualization of the learned features from (a) STIL and (b) \ourmethod{} on FF++. We highlight that \ourmethod{} achieves a more natural separation between real and fake samples, demonstrating improved generalization, while STIL overfits to seen fake ones with limited generalizability.}
    \label{vaoss}
\end{figure}

%% file: floats/fig_pca.tex
\begin{figure}[!t]
    \centering
    \subcaptionbox{STIL\label{fig:pca:stil}}[0.49\linewidth]
    {\includegraphics[width=\linewidth]{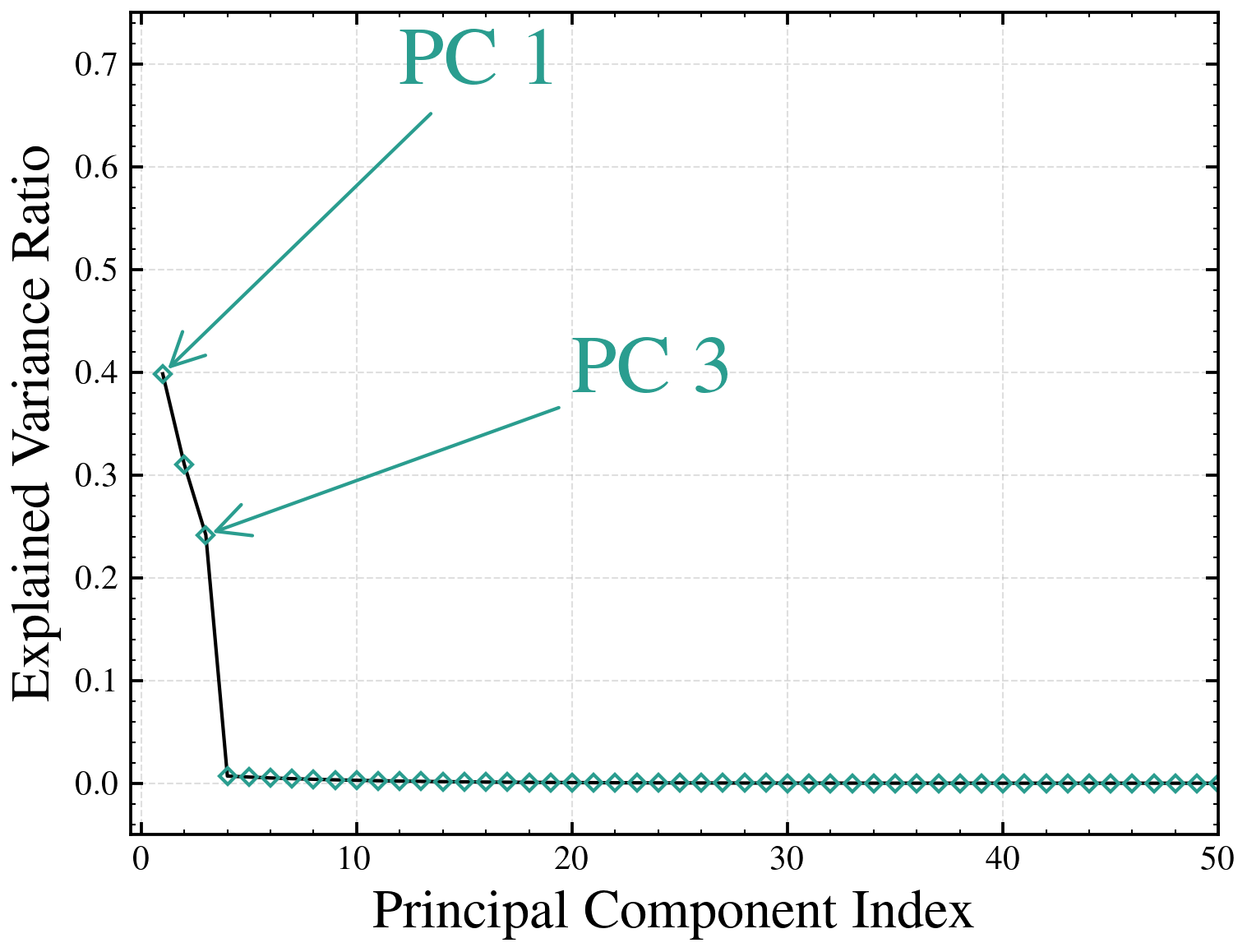}}
    \hfil
    \subcaptionbox{SSTGNN(ours)\label{fig:pca:sstgnn}}[0.49\linewidth]
    {\includegraphics[width=\linewidth]{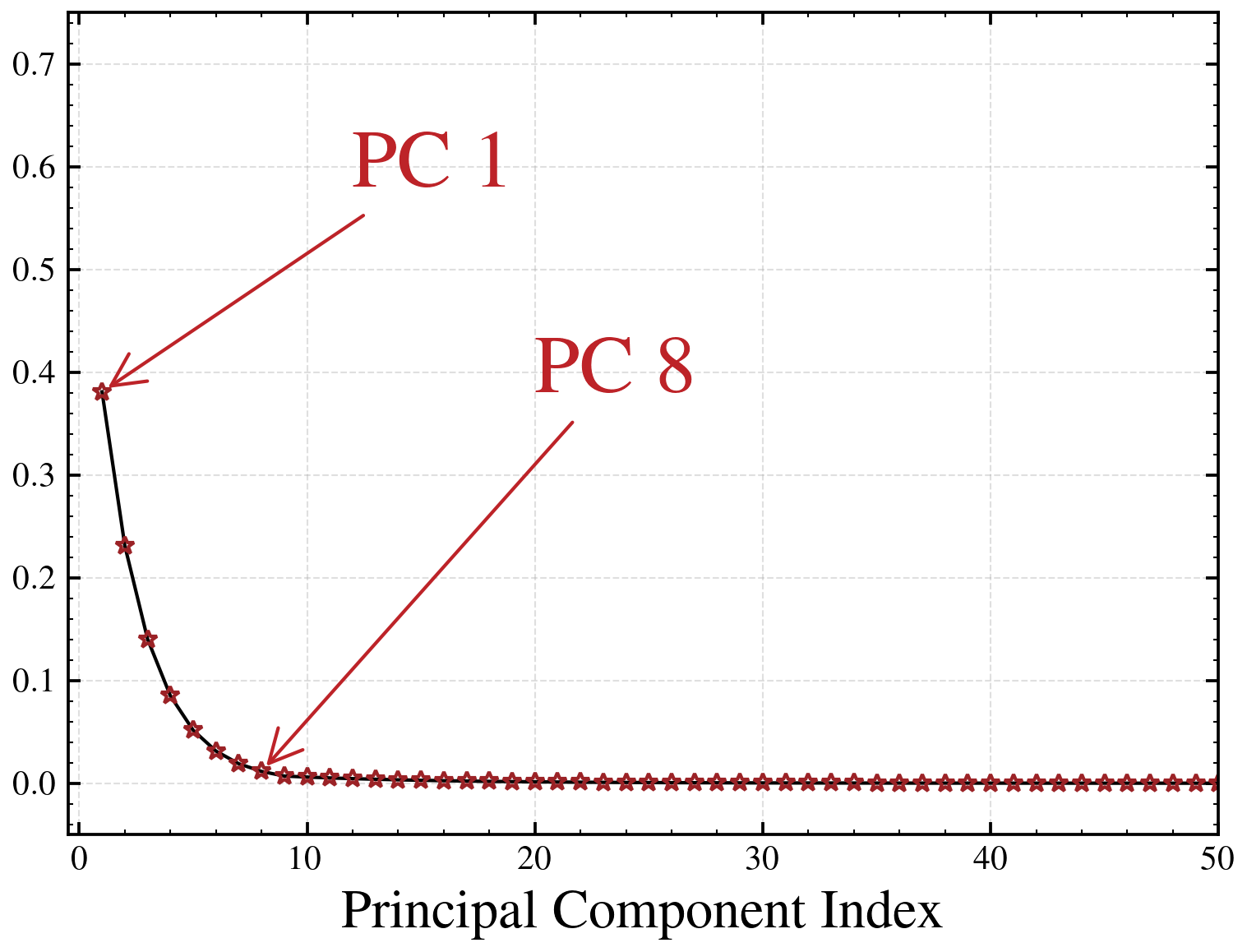}}
    \vspace{-0.3cm}
    \caption{PCA-based analysis of feature space expressiveness on FF++. For STIL, over 90\% of the variance is explained by only top\emph{-Three} principal components, reflecting a low-rank and limited representation. In comparison, \ourmethod{} spreads the variance across the first top\emph{-Eight} components, indicating a more expressive and informative feature space.}
    \label{vaoss2}
    \vspace{-0.3cm}
\end{figure}

%% file: floats/tab_resource.tex
\begin{table}[!t]
\centering
\caption{Resource allocation comparison on FF++ (DF, F2F). Training time(s), inference time(ms), and memory (MB).The row Saved indicates the times of resource consumption saved of SSTGNN ({\color{red}\textbf{Red}}) to the runner-up ({\color{blue}\textbf{Red}}) method.}
\vspace{-0.2cm}
\setlength{\tabcolsep}{5.6pt}
\resizebox{0.46\textwidth}{!}{
\renewcommand{\arraystretch}{1.2}
\begin{tabular}{l|cc|cc|cc}
\toprule
\multirow{2}{*}{\textbf{Methods}} & \multicolumn{2}{c|}{\textbf{Training}}
& \multicolumn{2}{c|}{\textbf{Inference}} 
& \multicolumn{2}{c}{\textbf{Memory}} \\
\cmidrule(lr){2-3} \cmidrule(lr){4-5} \cmidrule(lr){6-7}
 & DF & F2F & DF & F2F & DF & F2F \\
\midrule
STIL                 & \textcolor{blue}{\textbf{5986}} & \textcolor{blue}{\textbf{6043}} & \textcolor{blue}{\textbf{12}} & 12 & \textcolor{blue}{\textbf{18758}} & \textcolor{blue}{\textbf{18756}} \\
DCL                  & 6236 & 6348 & \textcolor{blue}{\textbf{12}} & 13 & 24312 & 24334 \\
MINTIME              & 14760 &15326 &23 & 21 & 22964 & 23067 \\
TALL                 & 7210 & 7367 & 14 & 14 & 28351 & 28439 \\
\midrule
\textbf{\ourmethod{}} & \textcolor{red}{\textbf{4698}} & \textcolor{red}{\textbf{4746}} & \textcolor{red}{\textbf{10}} & \textcolor{red}{\textbf{11}} & \textcolor{red}{\textbf{8546}} & \textcolor{red}{\textbf{8540}} \\
\textbf{Saved ($\times$)} & 
\textcolor{black}{\textbf{1.27}} & 
\textcolor{black}{\textbf{1.27}} & 
\textcolor{black}{\textbf{1.20}} & 
\textcolor{black}{\textbf{1.09}} & 
\textcolor{black}{\textbf{2.20}} & 
\textcolor{black}{\textbf{2.20}} \\
\bottomrule
\end{tabular}}
\label{tab:efficiency}
\end{table}

%% file: floats/tab_ablation.tex
\begin{table}[!t]
\centering
\vspace{-0.1cm}
\caption{Ablation: spatial, spectral and temporal components.}
\vspace{-0.2cm}
\setlength{\tabcolsep}{6.7pt}
\resizebox{0.46\textwidth}{!}{
\renewcommand{\arraystretch}{1.2}
\begin{tabular}{ccc|c|c}
\toprule \textbf{Temporal} & \textbf{Spectral} & \textbf{Spatial} & \textbf{Accuracy } & \textbf{AUC} \\
\midrule
& \checkmark & \checkmark &  95.08 &  98.30\\
\checkmark &           & \checkmark & 93.78 & 97.33 \\
\checkmark & \checkmark &           & 92.90 & 98.03 \\
\midrule
\checkmark & \checkmark & \checkmark & \textcolor{red}{\textbf{95.63}} & \textcolor{red}{\textbf{98.76}} \\
\bottomrule
\end{tabular}}
\label{abla1}
\end{table}

%% file: floats/fig_ablation.tex
\begin{figure}[!t]
    \centering
    \vspace{0.3cm}
    \subcaptionbox{Patch Size\label{fig:ablation:ps}}[0.49\linewidth]
    {\includegraphics[width=\linewidth]{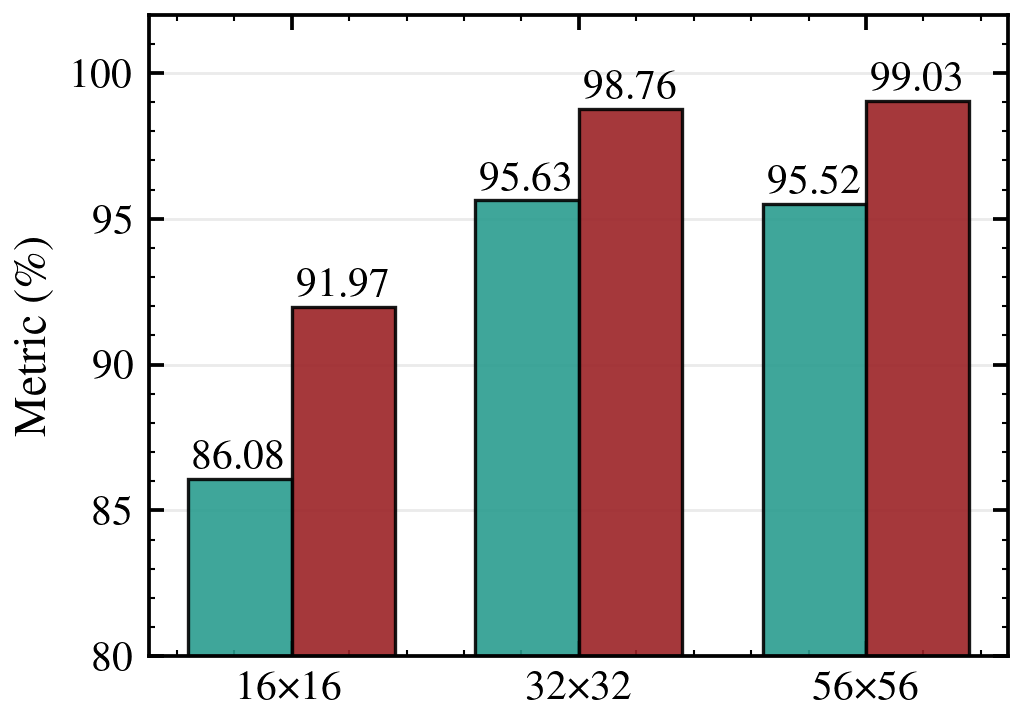}}
    \hfil
    \subcaptionbox{Threshold $\tau$\label{fig:ablation:threshold}}[0.49\linewidth]
    {\includegraphics[width=\linewidth]{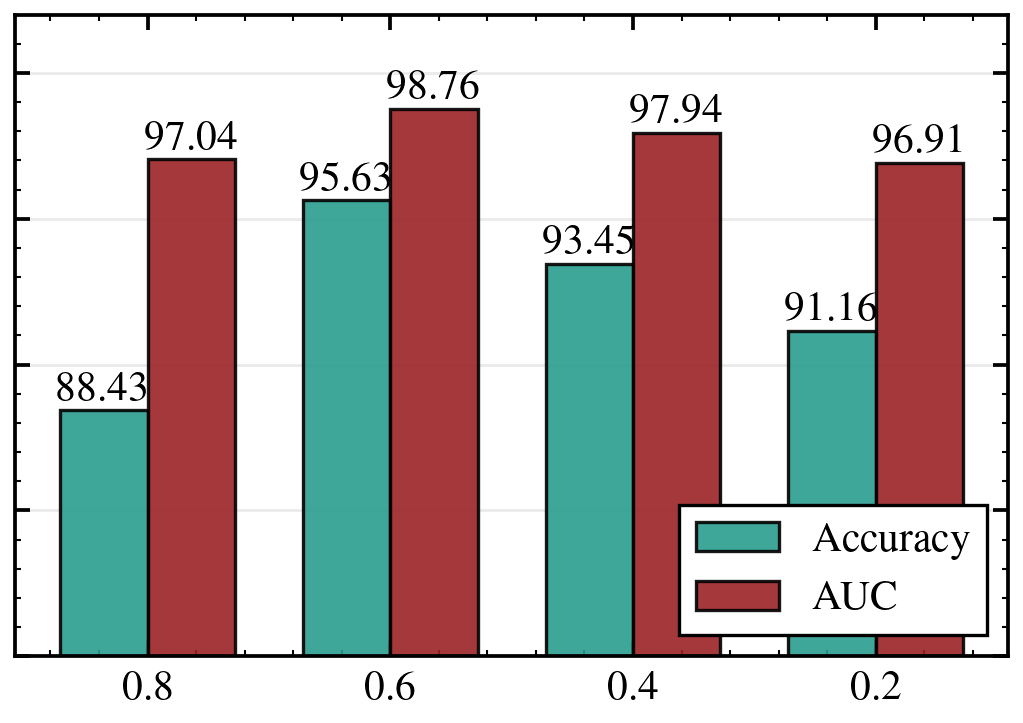}}
    \vspace{-0.3cm}
    \caption{Ablation study on patch size and threshold $\tau$.}
    \label{fig:ablation}
\end{figure}

%% file: floats/fig_visualization.tex
\begin{figure}[!t]
    \centering
    \vspace{-0.3cm}
    \includegraphics[width=0.97\linewidth]{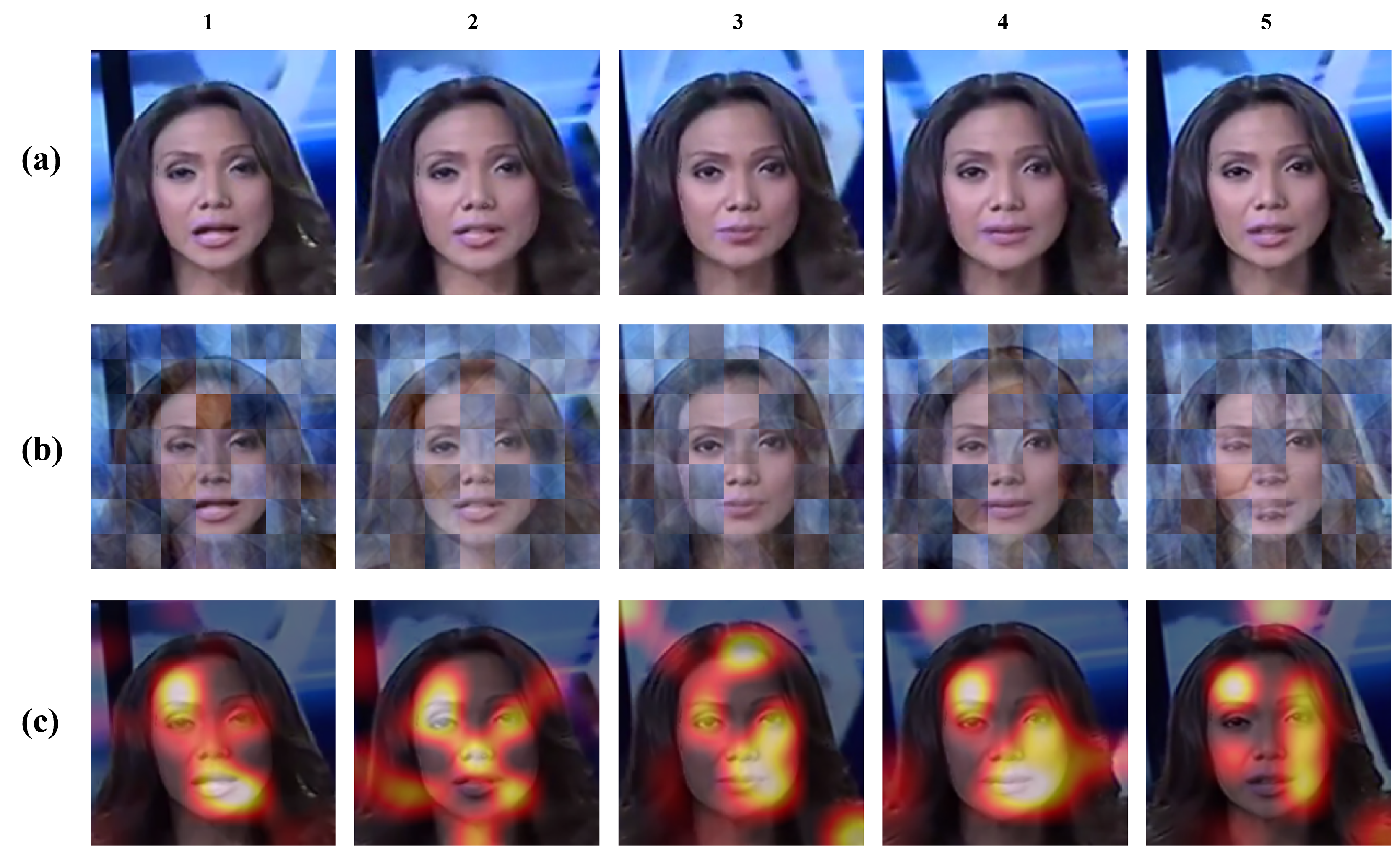}
    \caption{Interpretability Visualizations: (a) raw video, (b) patch-level spectral response and (c) patch-level heatmaps.}
    \label{fig:visualization}
\end{figure}

%% file: text/conclusion.tex
\section{Conclusion}
This paper presents \ourmethod{}, a unified framework that integrates spatial, temporal, and spectral information for robust deepfake detection. By modeling videos as patch-level graphs and incorporating learnable spectral filters with temporal differential modeling, \ourmethod{} captures subtle manipulation artifacts that are difficult to detect using conventional appearance-based approaches. Extensive experiments on diverse datasets demonstrate strong in-domain and cross-domain generalization while requiring significantly fewer parameters than existing baselines. 
Through joint reasoning over spatial, spectral, and temporal inconsistencies within a lightweight graph-based architecture, \ourmethod{} not only improves detection accuracy but also enables interpretable analysis of manipulation-related structural patterns. These characteristics make \ourmethod{} well suited for deployment in resource-constrained and open-world scenarios, where both robustness and efficiency are critical.
Looking forward, SSTGNN provides a flexible foundation that may be extended to other video forensics and multimedia integrity tasks involving complex spatial-temporal dependencies.

%% file: text/acks.tex
\section*{ACKNOWLEDGMENTS}
This research is supported by the National Research Foundation, Singapore under its AI Singapore Programme (AISG Award No: AISG3-RP-2024-034).

%% file: text/appendix.tex
\appendix
\section{Appendix}
\subsection{Procedure of \ourmethod{}}
\label{sec:algo}
In Algorithm~\ref{alg:SSTGNN}, SSTGNN first partitions each video into \(N\) patches and extracts patch embeddings using the feature encoder \( \theta \). Intra-frame adjacency matrices \( \mathbf{A}_t \) are then constructed based on similarities between normalized patch features, forming frame-level graphs \( \mathbf{G}_t \), which are subsequently connected across consecutive frames via temporal adjacency. After pruning edges using thresholds \( \tau_s \) and \( \tau_t \), a unified spatiotemporal graph \( \mathbf{G} \) is obtained. Spectral filters are learned by applying an MLP to the eigenvalues of the graph Laplacian to generate spectral features. To model spatial and temporal differentials, sub-adjacency matrices are constructed and augmented with negative edges. The resulting spatial and spectral features are aggregated using graph attention layers and concatenated to produce the final prediction for deepfake video classification. Overall, SSTGNN efficiently integrates graph-based learning with adaptive spectral filtering to capture subtle spatiotemporal inconsistencies across diverse deepfake manipulations.

\input{floats/algo_sstgnn}
\input{floats/fig_cross_1-3}

\subsection{Detail of Datasets}\label{app1}

\begin{table}[!b]
\Huge
\centering
\vspace{-0.2cm}
\caption{Video Datasets Overview.}
\vspace{-0.2cm}
\label{datasets}
\renewcommand{\arraystretch}{1.2}
\resizebox{0.48\textwidth}{!}{
\begin{tabular}{lcccccc}
\toprule
\multirow{2}{*}{\textbf{Video source}} & 
\multirow{2}{*}{\textbf{Type}} &  
\multirow{2}{*}{\textbf{FPS}} &
\multirow{2}{*}{\textbf{Length}} & 
\multirow{2}{*}{\textbf{Resolution}} & 
\multicolumn{2}{c}{\textbf{Size}} \\
 & & & & & \textbf{Real} & \textbf{Fake} \\
\midrule
FF++ \cite{rossler2019faceforensics}  
& I2V \& V2V 
& $\leq$30 
& 5-15s 
& $\leq$1280$\times$720 
& 1,000 
& 4,000 \\

Celeb-DF (CD) \cite{li2020celeb} 
& V2V 
& $\leq$30 
& 13s
& $\leq$1280$\times$720 
& 590 
& 5,639 \\

Wild-DF \cite{zi2020wilddeepfake}
& V2V 
& $\leq$25 
& 5-15s
& $\leq$854$\times$480 
& 707 
& 707 \\

DFDC \cite{dolhansky2019deepfake}
& V2V 
& 30 
& 3-15s
& $\leq$1280$\times$720 
& 2,000 
& 3,000 \\

SEINE \cite{chen2023seine}
& I2V 
& 8 
& 2-4s
& 1024$\times$576 
& - 
& 24,737 \\

SVD \cite{blattmann2023stable}
& I2V 
& 8 
& 4s
& 1280$\times$720 
& - 
& 149,026\\

Pika \cite{pikaart2022}
& T2V \& I2V 
& 24 
& 3s
& 1088$\times$640 
& - 
& 98,377\\

OpenSora \cite{peng2025open}
& T2V \& I2V 
& 24 
& 3s
& 1088$\times$640 
& - 
& 177,410\\

ZeroScope \cite{zeroscope2024}
& T2V  
& 8 
& 3s
& 1024$\times$576 
& - 
& 133,169\\

Crafter \cite{chen2023videocrafter1}
& T2V  
& 8 
& 4s
& 256$\times$256 
& - 
& 1,400\\

Gen2 \cite{esser2023structure}
& T2V \& I2V  
& 24 
& 4s
& 896$\times$512 
& - 
& 1,380\\

Lavie \cite{wang2025lavie}
& T2V  
& 8
& 2s
& 1280$\times$2048 
& - 
& 1,400\\

MSR-VTT \cite{xu2016msr}
& -  
& 30
& 10-30s
& 320$\times$240
& 10,000 
& -\\

Youku-mPLUG \cite{xu2023youku}
& -  
& 25-30
& 10-120s
& $\leq$1920$\times$1080
& 10,000 
& -\\
\bottomrule
\end{tabular}}
\end{table}

Table~\ref{datasets} summarizes the video datasets used in our study, covering a broad spectrum of real and fake sources, including traditional deepfake benchmarks, modern generative video datasets, and large-scale real-world corpora. For facial forgery detection, we include classic datasets such as FF++ and Celeb-DF, which rely on video-to-video (V2V) or image-to-video (I2V) manipulations (e.g., DeepFakes and FaceSwap) with moderate resolution and short clip lengths. Wild-DF and DFDC further introduce more realistic variations in lighting, occlusion, and subject diversity. Recent large-scale generative datasets—including SEINE, SVD, and Pika—reflect advances in image- and text-driven video generation, offering high-quality outputs and diverse content. To assess generalization to unseen generative patterns, we adopt open-world benchmarks such as OpenSora, ZeroScope, Crafter, Gen2, and Lavie, which are derived from emerging text-to-video (T2V) models and vary in frame rates, resolutions, and generation styles. In particular, our 3-to-3 evaluation trains on SEINE, SVD, and Pika and tests on Crafter, Gen2, and Lavie. For real-world reference, we use MSR-VTT and Youku-mPLUG as clean real samples to benchmark false positives and open-set robustness. Overall, our dataset suite comprises 13 datasets spanning V2V, I2V, and T2V paradigms, with durations from 2 to 120 seconds, resolutions from 240p to 2K, and over 500,000 fake videos and 20,000+ real clips, enabling robust in-domain and cross-domain evaluation.

\subsection{Additional Cross-Domain Results}\label{app4}

Figure~\ref{cross3} presents results under a more challenging 1-to-many generalization setting, where models are trained on a single dataset \textit{SVD} and tested on other complex generative video datasets: Crafter, Gen2, and Lavie. These datasets go beyond traditional facial forgeries and introduce more diverse generation artifacts, making this evaluation a stronger test of generalization. As shown in Figure~\ref{cross3}, our method consistently demonstrates top-tier performance. 
When trained on SVD, our method secures the highest overall average (94.33\% accuracy and 98.43\% AUC). It consistently ranks either first or second across all test datasets, achieving notable gains on Gen2 (95.31\%) and Lavie (92.9\%). This is especially impressive given the increased diversity and complexity in video generation across domains. Compared to strong baselines like STIL and DCL, our method shows more balanced performance across all targets, whereas others tend to perform well on one dataset but degrade on others.

\section{Proof of Theorem~\ref{lem:NPR}}
\label{appendix:proof_NPR}
\begin{proof}
We aim to prove that the $\ell_0$-NPR technique is equivalent to our spatial differential module under the assumption that the patch size $\ell_0 = 1$ and the Simplified Graph Convolution (SGC) aggregation method is used. The original $\ell_0$-NPR technique applies a local differential operation to pixel grids of size $\ell_0 \times \ell_0$. Specifically, for each pixel $w_{i,j}$ in the grid, the transformation is performed by subtracting the value of the top-left pixel, $w_{1,1}$, such that
\(
\widetilde{w}_{i,j} = w_{i,j} - w_{1,1}, \forall 1 \leq i, j \leq \ell_0.
\)
This can be expressed in matrix form as
\[
\mathbf{W} \rightarrow \widetilde{\mathbf{W}} = \mathbf{W} - \mathbf{W}_{1,1} \mathbf{1},
\]
where $\mathbf{W}$ is the matrix of pixel values and $\mathbf{W}_{1,1}$ denotes the top-left pixel in the grid, and $\mathbf{1}$ is a matrix of ones of the same dimension as $\mathbf{W}$. This transformation effectively captures local interdependence among pixels and is particularly useful in detecting manipulated regions within images. Now, we generalize this $\ell_0$-NPR process to graph embeddings. For a node embedding $v^{(t)}_{i_0,j_0}$ with coordinates $i_0 = i \ell_0$ and $j_0 = j \ell_0$ for $1 \leq i, j \leq \lceil N/\ell_0 \rceil$, we construct $\mathbf{\widetilde{A}}_{i,j}$ as:
\[
\mathrm{diag}(\mathbf{\widetilde{A}}_{i,j}) = \mathbf{1}, \quad \mathbf{\widetilde{A}}_{i,j}(v^{(t)}_{i_0, j_0}, :) = -\mathbf{1}, \quad \mathbf{\widetilde{A}}_{i,j}(:, v^{(t)}_{i_0}) = -\mathbf{1}.
\]
This matrix is used for message passing in the graph learning process, where each node aggregates information from its neighboring nodes according to the adjacency structure. The introduction of negative edges in the adjacency matrix induces a local differential effect, which is conceptually related to the $\ell_0$-NPR method that subtracts a reference pixel value (the top-left pixel) from other pixels within a local grid. In graph learning, message passing is commonly implemented using aggregation schemes such as Simplified Graph Convolution (SGC), where the update rule for a node $v$ at iteration $t$ is given by
\(
h_v^{(t+1)} = \mathbf{A} h_v^{(t)},
\)
with $\mathbf{A}$ denoting the adjacency matrix and $h_v^{(t)}$ the feature vector of node $v$ at time step $t$. Through this operation, each node incorporates information from its local neighborhood to refine its representation. In our formulation, the message passing matrix $\widetilde{\mathbf{A}}_{i,j}$ contains negative edge weights, which attenuate or subtract the contributions of neighboring nodes. This mechanism approximates a local differential operation, as feature updates emphasize relative differences rather than absolute values. When $\ell_0 = 1$, the patch size degenerates to a single pixel, and the resulting update simplifies to a difference between a node’s feature and those of its immediate neighbors. This behavior closely matches the neighborhood-level adjustment performed by $\ell_0$-NPR, where each pixel value is modified by subtracting a fixed local reference.

Therefore, under the specific setting of $\ell_0 = 1$ and SGC-style aggregation, the proposed spatial differential module can be regarded as functionally equivalent to the $\ell_0$-NPR operation in terms of applying a local differential transformation. Both approaches capture local interdependence by emphasizing relative differences within a small neighborhood, leading to the following equivalence:
\[
\ell_0\text{-NPR} \;\equiv\; \text{Spatial Differential (SGC aggregation, } \ell_0 = 1\text{)}.
\]
\end{proof}

%% file: floats/algo_sstgnn.tex
\begin{algorithm}[!h]
\caption{SSTGNN model}
\label{alg:SSTGNN}
\begin{algorithmic}[1]
\REQUIRE 
\( \mathcal{F} \in \mathbb{R}^{T \times H \times W \times C} \), thresholds \( \tau_s \) and \( \tau_{t} \)
\ENSURE Prediction results $\mathbf{y} \in \mathbb{R}^{N \times 2}$
\STATE Divide frames into $N$ patches and extract embeddings via $g_{\theta}$
\STATE Calculate intra-frame edge $\mathbf{A}^{(t)} = \mathrm{SIM}(\mathbf{\overline{X}}^{(t)}, \mathbf{\overline{X}}^{(t)})$, and construct frame-level subgraph $\textbf{G}^{(t)}=(\textbf{V}^{(t)},\textbf{A}^{(t)})$
\STATE Bridge subgraphs $\textbf{G}^{(t)}$ and $\textbf{G}^{(t+1)}$ as $\mathbf{S}^{(t)}(v^{(t)}_i, v^{(t+1)}_i)$
\STATE Form the graph $\mathbf{G} = (\mathbf{V}, \mathbf{A}, \mathbf{\overline{A}}, \mathbf{X})$ after applying \( \tau_s \) and \( \tau_{t} \) to prune redudent edges
\STATE Learn the spectral filter $\boldsymbol{\phi}(\boldsymbol{\lambda}) = f_{\mathrm{MLP}}(\boldsymbol{\lambda}) \in \mathbb{R}^N$, and yield spectral features $\mathbf{Z}_{\text{spectral}}$
\STATE Construct the sub-adjacency matrices $\widetilde{A}$ with negative edges
\STATE Concatenate features from adjacent frames and add negative edge $\mathbf{\overline{A}}(v^{(t)}_{i,j}, v^{(t+1)}_{i,j}) = -1$ into inter-frame adjacency matrix 
\STATE Apply GAT to video graph and learn the spatial features  
\STATE Concatenate features \( \mathbf{Z} = \mathrm{Concat}[\mathbf{Z}_{\text{spatial}}\Vert \mathbf{Z}_{\text{spectral}}] \) and learn towards the prediction results
\end{algorithmic}
\end{algorithm}




%% file: floats/fig_cross_1-3.tex
\begin{figure*}[!t]
    \centering
    \vspace{-0.15cm}
    \includegraphics[width=0.955\linewidth]{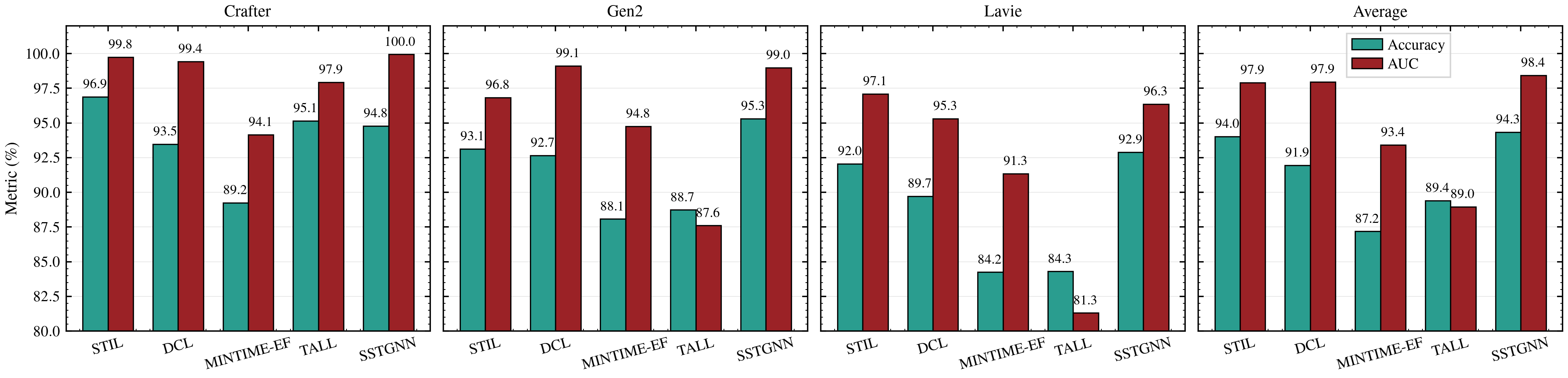}
    \vspace{-0.4cm}
    \caption{Cross-dataset performance of different methods under the 1-to-many setting by training on \textit{SVD}. We report both Accuracy (solid bars) and AUC (bars with hatch) for each method. \textit{MIN} denotes \textit{MINTIME-EF} in the figure.}
    \label{cross3}
    \vspace{-0.4cm}
\end{figure*}

